%% file: multimodal.tex
\documentclass[twoside,11pt]{article}

\usepackage{booktabs}
\usepackage{graphicx}
\usepackage{tabularx}
\usepackage{subcaption}
\usepackage{algorithmic}
\usepackage{algorithm}
\usepackage[preprint]{jmlr2e}

\input{math_commands.tex}

\newtheorem{thm}{Theorem}
\newtheorem{lem}{Lemma}

\newtheorem{coro}{Corollary}

\graphicspath{ {./images/} }

\ShortHeadings{Multimodal Generative Models for Compositional Representation Learning}{Wu and Goodman}
\firstpageno{1}

\begin{document}

\title{Multimodal Generative Models for Compositional Representation Learning}

\author{\name Mike Wu \email wumike@stanford.edu \\
       \addr Department of Computer Science\\
       Stanford University\\
       Stanford, CA 94025, USA
       \AND
       \name Noah Goodman \email ngoodman@stanford.edu \\
       \addr Department of Computer Science and Psychology\\
       Stanford University\\
       Stanford, CA 94025, USA}

\maketitle

\begin{abstract}
As deep neural networks become more adept at traditional tasks, many of the most exciting new challenges concern multimodality---observations that combine diverse types, such as image and text.
In this paper, we introduce a family of multimodal deep generative models derived from variational bounds on the evidence (data marginal likelihood).
As part of our derivation we find that many previous multimodal variational autoencoders used objectives that do not correctly bound the joint marginal likelihood across modalities.
We further generalize our objective to work with several types of deep generative model (VAE, GAN, and flow-based), and allow use of different model types for different modalities.
We  benchmark our models across many image, label, and text datasets, and find that our multimodal VAEs excel with and without weak supervision.
Additional improvements come from use of GAN image models with VAE language models.
Finally, we investigate the effect of language on learned image representations through a variety of downstream tasks, such as compositionally, bounding box prediction, and visual relation prediction.
We find evidence that these image representations are more abstract and compositional than equivalent representations learned from only visual data.
\end{abstract}

\section{Introduction}

The evidence that reaches our senses is an indirect reflection of complex latent structure in the world.
Objects are composed from pieces and substances, they interact with each other in complex ways, and are acted on by agents who have complex goals and beliefs.
Because different modalities---vision, touch, sound, language---reflect these latent causes in different ways, models that learn from multiple modalities should yield richer, more abstract, and more generalizable representations.
For instance, language is intrinsically compositional, reflecting complex properties and relations. Vision does not directly reflect this structure, spreading the same causes across pixels, but provides far more continuous detail.
This suggests the hypothesis that visual representations learned from both visual and linguistic data would reflect the abstract compositional structure of objects more directly.
In this paper we explore a new family of multimodal generative models and show that it efficiently learns from complex multimodal data, including from weak supervision, and appears to reflect more abstract and compositional information.



The bulk of previous work on multimodal learning has focused on supervised approaches using deep neural networks \citep{ngiam2011multimodal,xu2015show,liu2018learn,baruch2018multimodal}. While this family of models can work well, their performance is predicated on large curated datasets of \textit{paired} examples from different modalities---requiring all the modalities describing each observation to be present.
This can be almost impossible in many domains where data is either rare or prohibitively expensive.
For example, collecting tactile data requires specialized equipment and is difficult to do at the scale of ImageNet or even COCO. Under these stringent restrictions of data size, supervised multimodal approaches may be suspect to overfitting and poor generalization.

Generative models specialize in \textit{un-}supervised learning, and deal naturally with missing data (such as missing modalities).
There have recently been a number of multimodal generative models \citep{wang2016deep,suzuki2016joint,vedantam2017generative,wu2018multimodal} built on the variational autoencoder \citep{kingma2013auto}. Some of these multimodal VAEs enable weakly-supervised learning, where the model only needs a small number of paired examples to learn the ``alignment" between different modalities learned by using unpaired data.
In many domains it is often the case that while paired examples are expensive, \textit{unpaired} examples (a single modality) are cheap.
For instance, one can find an unlimited amount of unrelated images and text online.
However, previous experiments have been limited to straightforward domains with simple images (e.g. MNIST) and labels. These models struggle with richer domains such as natural language or more naturalistic images (e.g. CIFAR).
In this paper we argue that there are two reasons for this: the objectives previously used were often not correct bounds on the joint probability, and, individual modalities did not allow strong enough generative models to capture the natural variation in the data.

In this work, we introduce a new family of objectives for multimodal generative models, derived from minimizing variational divergences. We show that many of the previous multimodal VAEs are either special cases of our family or have undesirable objectives that over-emphasize either the marginal or conditional densities. We show that our new multimodal VAE is more stable than and outperforms prior works. Then, to handle richer images (which VAEs notoriously \citep{zhao2017infovae} fail to capture), we show that our multimodal objective is well-defined when individual modalities are parameterized by VAEs, GANs, or invertible flows. Doing so allows us to combine the benefits across these different models: using VAEs for discrete domains such as text and GANs/flows for complex images. We then thoroughly investigate how our family of models performs under weak supervision, varying both the amount of paired and unpaired data.
Finally, we explore the representations learned, finding that visual representations co-learned with language are more abstract: they are more compositional and better support transfer tasks of object detection and visual relation prediction.

\section{Background}
We briefly review popular unimodal generative models.

\subsection{Variational Autoencoders}
\label{sec:vae}
A variational autoencoder (VAE) \citep{kingma2013auto,rezende2014stochastic} specifies a joint distribution $p_\theta(\vx,\vz) = p(\vz)p_\theta(\vx|\vz)$ over a set of observed variables $\vx \in \mathcal{X}$ and stochastic variables $\vz \in \mathcal{Z}$ parameterized by $\theta \in \Theta$. The prior, $p(\vz)$ is often isotropic Gaussian whereas the likelihood, $p_\theta(\vx|\vz)$ is usually Bernoulli for images and Categorical for language. Instead of computing the (intractable) posterior distribution $p_\theta(\vz|\vx)$, VAEs introduce an approximate posterior $q_\phi(\vz|\vx) \in \mathcal{Q}$ from a variational family of ``simple" distributions that we can score and sample from (e.g. Gaussians). Letting $p_d(\vx)$ represent an empirical distribution from which we sample our training dataset, we define a variational joint distribution as $q_\phi(\vx,\vz) = p_d(\vx)q_\phi(\vz|\vx)$. We can interpret VAEs as minimizing the Kullback-Liebler divergence, denoted $D_{\text{KL}}$, between the two joint distributions:
\begin{align}
\KL(q_\phi(\vx,\vz)||p_\theta(\vx,\vz)) &= \mathbb{E}_{q_\phi(\vx,\vz)}\left[\log \frac{q_\phi(\vx,\vz)}{p_\theta(\vx,\vz)}\right] \nonumber \\
&\propto \mathbb{E}_{p_d(\vx)}\mathbb{E}_{q_\phi(\vz|\vx)}\left[\log \frac{q_\phi(\vz|\vx)}{p_\theta(\vx,\vz)}\right]\nonumber \\
&= \mathbb{E}_{p_d(\vx)}[-\text{ELBO}(\vx)]
\label{eqn:vae_objective}
\end{align}

Minimizing Eq.~\ref{eqn:vae_objective} is equivalent to maximizing the evidence lower bound (ELBO) as commonly presented in VAE literature. In practice, $q_\phi(\vz|\vx)$ and $p_\theta(\vx|\vz)$ consist of deep neural networks, and Eq.~\ref{eqn:vae_objective} is optimized with stochastic gradient descent using the reparameterization trick to minimize variance when estimating the gradient.

\subsection{Generative Adversarial Networks}
\label{sec:gan}
As elegant as they are, VAEs are not very adept with naturalistic images, as evidenced by blurry samples \citep{zhao2017infovae,zhao2017towards,gulrajani2016pixelvae}. In response, newer generative models focus heavily on sample quality. In particular, a generative adversarial network (GAN) \citep{goodfellow2014generative} is a family of deep generative models that learns a deterministic transformation, $f_\theta: \mathcal{Z} \rightarrow \mathcal{X}$, on samples from a prior distribution $p(\vz)$ along with an adversarial model trained to distinguish between transformed samples $f_\theta(\vz)$ and empirical samples drawn from $p_d(\vx)$. The two models are jointly trained in a mini-max game. Unlike VAEs, GANs do not explicitly learn a distribution over $\vx$ (i.e. likelihood-free) but allow us to sample from an implicit distribution, $p_\theta(\vx|\vz)$. As in \citep{nowozin2016f,dumoulin2016adversarially}, we rewrite the minimax objective as a Fenchel dual of a $f$-divergence (with KL as the metric).
\begin{align}
\KL(q_\phi(\vx,\vz)||p_\theta(\vx,\vz)) &\geq \sup_{T \in \mathcal{T}}\left(\mathbb{E}_{q_\phi(\vx,\vz)}[T(\vx,\vz)] - \mathbb{E}_{p_\theta(\vx,\vz)}[e^{T(\vx,\vz) - 1}]\right) \nonumber \\
&\geq \sup_{T \in \mathcal{T}_\psi}\left(\mathbb{E}_{q_\phi(\vx,\vz)}[T(\vx,\vz)] - \mathbb{E}_{p_\theta(\vx,\vz)}[e^{T(\vx,\vz) - 1}]\right)
\label{eqn:gan}
\end{align}
where $p_\theta(\vx,\vz) = p_\theta(\vx|\vz)p(\vz)$ and $\mathcal{T}$ specifies a class of functions $T: \mathcal{X} \times \mathcal{Z} \rightarrow \mathbb{R}$. This is lower bounded by a smaller class of functions $\mathcal{T}_\psi$ parameterized by neural network parameters, $\psi$; we often call $T \in \mathcal{T}_\psi$ a discriminator or critic. Like in VAEs, the joint $q_\phi(\vx, \vz)$ is implicitly defined by an empirical distribution $p_d(\vx)$ and $q_\phi(\vz|\vx)$, an approximate inference network. Since we can sample from $p_\theta(\vx,\vz)$ using $\vx = f_\theta(\vz)$ where $\vz \sim p(\vz)$, we can optimize Eq.~\ref{eqn:gan} with SGD and reparameterization.

\subsection{Flow-based Models}
\label{sec:flow}
Flow-based models \citep{dinh2016density,papamakarios2017masked,kingma2018glow,huang2018neural} are another class of generative models that have been shown to work well with naturalistic images. Similar to GANs, flows learn a deterministic transformation $f_\theta: \mathcal{Z} \rightarrow \mathcal{X}$ but with the additional constraint that $f_\theta$ must be invertible and possess a tractable Jacobian determinant. For example, many flows enforce the Jacobian matrix to be diagonal \citep{rezende2015variational} or triangular \citep{dinh2016density} such that the determinant has a closed form expression. Unlike both VAEs and GANs, satisfying these constraints allows us to do \emph{exact} inference using a change of variables transformation:
\begin{align}
  \log p(\vx) = \log p(\vz) + \log | \det \frac{\partial \vz}{\partial \vx^T} |
\end{align}
where $\vz = f_\theta^{-1}(\vx)$, and $p(\vz)$ is a simple prior. As a result, generating an observation requires first sampling $\vz \sim p(\vz)$, then computing $\vx = f_\theta(\vz)$. Symmetrically,  inference is the opposite: first sample $\vx \sim p_d(\vx)$, followed by $\vz = f_{\theta}^{-1}(\vx)$. Unlike GANs, there is no adversary or additional objective. In practice, $f_\theta$ is parameterized by a neural network.

\section{Methods}
\label{sec:methods}
We now introduce a family of multimodal generative models extending the VAE framework presented above. For simplicitly, we restrict to two modalities: Let $\vx \in \mathcal{X}$ and $\vy \in \mathcal{Y}$ denote observed variables (e.g. images and text) that describe the same natural phenomena. Assume we are given an empirical distribution $p_d(\vx,\vy)$ over paired images and text. Using a latent variable $\vz \in \mathcal{Z}$, we define a generative model over the joint distribution that factors as $p_\theta(\vx,\vy,\vz) = p_\theta(\vx|\vz)p_\theta(\vy|\vz)p(\vz)$ where $p(\vz)$ is a prior i.e. $\vx \perp \vy | \vz$. Further, we define three variational posteriors, one for the joint setting $q_\phi(\vx,\vy,\vz) = p_d(\vx,\vy)q_\phi(\vz|\vx,\vy)$, and two for the unimodal cases $q_\phi(\vx,\vz) = p_d(\vx)q_\phi(\vz|\vx)$ and $q_\phi(\vy,\vz) = p_d(\vy)q_\phi(\vz|\vy)$. We note the following property between the joint and unimodal posteriors.
\begin{lem} The multimodal and unimodal posteriors are related by marginalization:
\begin{equation*}
  \int_{\mathcal{X}} q_\phi(\vx,\vy,\vz) = q_\phi(\vy,\vz)\quad,\quad\int_{\mathcal{Y}} q_\phi(\vx,\vy,\vz) = q_\phi(\vx,\vz)
\end{equation*}
\label{lem:marg}
\end{lem}
\begin{proof}
  We show the first equality. The second will follow by an identical argument.
  \begin{align*}
    \int_{\mathcal{X}} q_\phi(\vx,\vy,\vz) &= \int_{\mathcal{X}} p_d(\vx,\vy)q_\phi(\vz|\vx,\vy)
    = p_d(\vy)\int_{\mathcal{X}} p_d(\vx|\vy)q_\phi(\vz|\vx,\vy) \\
    &= p_d(\vy)\int_{\mathcal{X}}q_\phi(\vz,\vx|\vy)
    = p_d(\vy)q_\phi(\vz|\vy)
    = q_\phi(\vy,\vz)
  \end{align*}
\end{proof}
At first, defining three variational posteriors might seem excessive. Truthfully, the most naive generalization of Eq.~\ref{eqn:vae_objective} to the multimodal setting would only require the joint posterior $q_\phi(\vx,\vy,\vz)$ by minimizing
\begin{equation}
\mathcal{L}_0 = \KL(q_\phi(\vx,\vy,\vz)||p_\theta(\vx,\vy,\vz))
\label{eqn:naive:melbo}
\end{equation}
which is a lower bound on $\log p_\theta(\vx,\vy)$ by Jensen's inequality. However, with a multimodal model, we often want to ``translate" between modalities. For instance, we might want to hallucinate conditional images and generate conditional captions. Purely optimizing Eq.~\ref{eqn:naive:melbo} does not facilitate such operations as we cannot do inference given a single modality (i.e. missing a modality). To address this, we will need the other two posteriors. We present the following theorem, which decomposes Eq.~\ref{eqn:naive:melbo} into a sum of two divergence terms: one for the marginal and one for the conditional.

\begin{thm}
Let $\vx \in \mathcal{X}$, $\vy \in \mathcal{Y}$ be two observed variables, and $\vz \in \mathcal{Z}$ be a latent variable. Let $p_\theta(\vx,\vy,\vz)$ denote a latent variable model parameterizing the joint distribution. Let $q_\phi(\vx,\vy,\vz)$, $q_\phi(\vx,\vz)$ and $q_\phi(\vy,\vz)$ denote three posteriors over multimodal and unimodal inputs. Define a set of distances between distributions:
\begin{align}
D_\alpha &= \KL(q_\phi(\vx,\vz)||p_\theta(\vx,\vz)) \nonumber \\
D_\beta  &= \KL(q_\phi(\vx,\vy,\vz)||q_\phi(\vx,\vz)p_\theta(\vy|\vx,\vz)) \nonumber \\
D_\gamma &= \KL(q_\phi(\vy,\vz)||p_\theta(\vy,\vz)) \nonumber \\
D_\eta &= \KL(q_\phi(\vx,\vy,\vz)||q_\phi(\vy,\vz)p_\theta(\vx|\vy,\vz)) \nonumber
\end{align}
Then, the sum of the distances is equal to twice $\mathcal{L}_0$.
\begin{equation}
\mathcal{L}_{\text{M}} = D_\alpha + D_\beta + D_\gamma + D_\eta = 2\mathcal{L}_0
\label{eqn:mm:objective}
\end{equation}
\label{thm:1}
\vspace{-2em}
\end{thm}
\begin{proof}
To show that the statements hold, we need to show both
\begin{align*}
\KL(q_\phi(\vx,\vy,\vz)||p_\theta(\vx,\vy,\vz)) &= \KL(q_\phi(\vx,\vz)||p_\theta(\vx,\vz)) + \KL(q_\phi(\vx,\vy,\vz)||q_\phi(\vx,\vz)p_\theta(\vy|\vx,\vz)) \\
\KL(q_\phi(\vx,\vy,\vz)||p_\theta(\vx,\vy,\vz)) &= \KL(q_\phi(\vy,\vz)||p_\theta(\vy,\vz)) + \KL(q_\phi(\vx,\vy,\vz)||q_\phi(\vy,\vz)p_\theta(\vx|\vy,\vz))
\end{align*}
We show only first equality since the other should be a symmetric argument.
We begin by writing the divergence explicitly,
\begin{align*}
  \KL(q_\phi(\vx,\vy,\vz)||p_\theta(\vx,\vy,\vz)) &= \int_{\mathcal{X}} \int_{\mathcal{Y}} \int_{\mathcal{Z}} q_\phi(\vx,\vy,\vz)\log\frac{q_\phi(\vx,\vy,\vz)}{p_\theta(\vx,\vy,\vz)}d\vz d\vy d\vx \\
  &= \int_{\mathcal{X}} \int_{\mathcal{Y}} \int_{\mathcal{Z}} q_\phi(\vx,\vy,\vz)\log\frac{q_\phi(\vx,\vz)}{p_\theta(\vx,\vz)}d\vz d\vy d\vx \\
  &\qquad + \int_{\mathcal{X}} \int_{\mathcal{Y}} \int_{\mathcal{Z}} q_\phi(\vx,\vy,\vz)\log\frac{q_\phi(\vy|\vx,\vz)}{p_\theta(\vy|\vx,\vz)}d\vz d\vy d\vx \\
  &= D_{1}(\vx,\vy,\vz) + D_{2}(\vx,\vy,\vz)
\end{align*}
We consider the two terms separately.
We will show that each term can be written as a KL divergence.
In the first line, use Lemma.~\ref{lem:marg} in the second equality.
\begin{align*}
  D_1(\vx,\vy,\vz) &= \int_{\mathcal{X}} \int_{\mathcal{Z}} \log\frac{q_\phi(\vx,\vz)}{p_\theta(\vx,\vz)} (\int_{\mathcal{Y}}q_\phi(\vx,\vy,\vz)d\vy)  d\vz d\vx =\int_{\mathcal{X}} \int_{\mathcal{Z}} \log\frac{q_\phi(\vx,\vz)}{p_\theta(\vx,\vz)}q_\phi(\vx,\vz) d\vz d\vx \\
  &= \KL(q_\phi(\vx,\vz)||p_\theta(\vx,\vz)) \\
  D_2(\vx,\vy,\vz) &= \int_{\mathcal{X}} \int_{\mathcal{Y}} \int_{\mathcal{Z}} q_\phi(\vx,\vy,\vz)\log\frac{q_\phi(\vx,\vy,\vz)}{q_\phi(\vx,\vz)p_\theta(\vy|\vx,\vz)}d\vz d\vy d\vx \\
  &= \KL(q_\phi(\vx,\vy,\vz)||q_\phi(\vx,\vz)p_\theta(\vy|\vx,\vz)
\end{align*}
\end{proof}
In Thm.~\ref{thm:1}, we propose $\mathcal{L}_{\text{M}}$ as an objective to optimize a multimodal generative model.
It remains to show that Eq.~\ref{eqn:mm:objective} is a valid lower bound on the multimodal evidence, $\log p_\theta(\vx,\vy)$.
We do so with the following Corollaries.
\begin{coro}
If $q_\phi(\vz|\vy)$ is a good approximation of $p_\theta(\vz|\vy)$, then $D_\beta \leq \log p_\theta(\vx|\vy)$. If $q_\phi(\vz|\vx)$ is a good approximation of $p_\theta(\vz|\vx)$, then $D_\eta \leq  \log p_\theta(\vy|\vx)$.
\label{lem:3}
\end{coro}
\begin{proof} We assume that $p_\theta(\vz|\vy)$ can be substituted by $q_\phi(\vz|\vy)$.
  \begin{align*}
    \log p_\theta(\vx|\vy) &\geq \mathbb{E}_{q_\phi(\vz|\vx,\vy)}[\log \frac{p_\theta(\vx|\vz)p_\theta(\vz|\vy)}{q_\phi(\vz|\vx,\vy)}] = \mathbb{E}_{q_\phi(\vz|\vx,\vy)}[\log \frac{p_\theta(\vx|\vz)q_\phi(\vz|\vy)}{q_\phi(\vz|\vx,\vy)}] \\
    &= \mathbb{E}_{q_\phi(\vz|\vx,\vy)}[\log p_\theta(\vx|\vz)q_\phi(\vz|\vy)] + D(q_\phi(\vz|\vx,\vy)||q(\vz|\vy)) \\
    &= \KL(q_\phi(\vy,\vz) || q_\phi(\vx,\vz)p_\theta(\vy|\vx,\vz)))
  \end{align*}
where line 1 follows by assumption and the last line follows by Lemma~\ref{lem:2}.
\end{proof}
As an immediate consequence, we get the following:
\begin{coro} Given $\mathcal{L}_{\text{M}}$ as defined in Eq.~\ref{eqn:mm:objective}, each component is a lower bound on a marginal or conditional distribution:
\begin{align}
D_\alpha &\leq \log p_\theta(\vx), &D_\beta \leq \log p_\theta(\vy) \label{eqn:corotwo:marg}\\
D_\gamma &\leq \log p_\theta(\vx|\vy), &D_\eta \leq \log p_\theta(\vy|\vx) \label{eqn:corotwo:cond}
\end{align}
Summed together, $\mathcal{L}_{\text{M}} = D_\alpha + D_\beta + D_\gamma + D_\eta$, this objective lower bounds double the joint distribution, $2\log p_\theta(\vx, \vy)$.
\end{coro}
\begin{proof}
Eq.~\ref{eqn:corotwo:cond} restates Coro.~\ref{lem:3}. See \citep{kingma2013auto} for a proof of Eq.~\ref{eqn:corotwo:marg}.
\end{proof}
Thus, optimizing $\mathcal{L}_{\text{M}}$ optimizes the marginals and conditionals as well as the joint density in a principled manner.
If we can explicitly evaluate the likelihoods (e.g. VAE), we can further simplify Thm.~\ref{thm:1} into an interpretable form.
\begin{lem} $D_\alpha, D_\beta, D_\gamma, D_\lambda$ as defined in Thm.~\ref{thm:1} can be written as a sum of four reconstruction terms and four Kullback-Liebner divergence terms.
\begin{align}
  \small
  D_\alpha &\propto \mathbb{E}_{p_d}[\mathbb{E}_{q_\phi(\vz|\vx)}[\log p_\theta(\vx|\vz)] + \KL(q_\phi(\vz|\vx)||p(\vz))] \label{eqn:lem1a}\\
  D_\gamma &\propto \mathbb{E}_{p_d}[\mathbb{E}_{q_\phi(\vz|\vy)}[\log p_\theta(\vy|\vz)] + \KL(q_\phi(\vz|\vy)||p(\vz))] \label{eqn:lem1b}\\
  D_\beta &\propto \mathbb{E}_{p_d}[\mathbb{E}_{q_\phi(\vz|\vx,\vy)}[\log p_\theta(\vx|\vz)] + \KL(q_\phi(\vz|\vx,\vy)||q_\phi(\vz|\vy))] \label{eqn:lem1c}\\
  D_\eta &\propto \mathbb{E}_{p_d}[\mathbb{E}_{q_\phi(\vz|\vx,\vy)}[\log p_\theta(\vy|\vz)] + \KL(q_\phi(\vz|\vx,\vy)||q_\phi(\vz|\vx))] \label{eqn:lem1d}
\end{align}
\label{lem:2}
\end{lem}
\begin{proof}
Eq.~\ref{eqn:lem1a},~\ref{eqn:lem1b} have  been shown in Sec.~\ref{sec:vae}. We will prove Eq.~\ref{eqn:lem1c} as the last is symmetric.
\begin{align*}
-D(q_\phi(\vy,\vz)||r_{\phi,\theta}(\vx,\vy,\vz)) &= -\int_{\mathcal{X}}\int_{\mathcal{Y}}\int_{\mathcal{Z}}q_\phi(\vx,\vy,\vz)\log\frac{q_\phi(\vx,\vy,\vz)}{q_\phi(\vx,\vz)p_\theta(\vy|\vx,\vz)}d\vz d\vy d\vx \\
&\propto -\int_{\mathcal{X}}\int_{\mathcal{Y}}p_d(\vx,\vy)\int_{\mathcal{Z}}q_\phi(\vz|\vx,\vy)\log\frac{q_\phi(\vz|\vx,\vy)}{q_\phi(\vz|\vx)p_\theta(\vy|\vx,\vz)}d\vz d\vy d\vx \\
&= \mathbb{E}_{p_d(\vx,\vy)}[\int_\vz q_\phi(\vz|\vx,\vy)\log\frac{q_\phi(\vz|\vx)p_\theta(\vy|\vx,\vz)}{q_\phi(\vz|\vx,\vy)}d\vz]\\
&= \mathbb{E}_{p_d(\vx,\vy)}[\mathbb{E}_{q_\phi(\vz|\vx,\vy)}[\log p_\theta(\vx|\vz)] + D(q_\phi(\vz|\vx,\vy)||q_\phi(\vz|\vy))]
\end{align*}
where in line 2, we factor out $\frac{p_d(\vx,\vy)}{p_d(\vx)}$ as a constant (which should not effect optimization); in line 4 we note $p_\theta(\vy|\vx,\vz) = p_\theta(\vy|\vz)$ since $\vx \perp \vy|\vz$ (consider the graphical stucture for the multimodal VAE).
\end{proof}
The first and third terms in Lemma~\ref{lem:2} are standard ELBOs for $\vx$ and $\vy$. The second and fourth terms each contain a KL divergence term pulling the unimodal variational posterior, $q_\phi(\vz|\cdot)$ closer to the multimodal posterior, $q_\phi(\vz|\vx,\vy)$. This intuitively regularizes the inference network when given only one observed modality. For example, minimizing $\KL(q_\phi(\vz|\vx,\vy)||q_\phi(\vz|\vy))$ can be interpreted as ``do inference as if $\vx$ was present when using $q_\phi(\vz|\vy)$". One can also interpret this as mapping the unimodal inputs $\vx$, $\vy$, and multimodal input $(\vx,\vy)$ to a shared embedding space. Most importantly, we note that by introducing these unimodal inference networks, we can now translate between modalities as desired. For example, we can now sample from $\vy \sim p_\theta(\vy|\vz)$ where $\vz \sim q_\phi(\vz|\vx)$ and vice versa.

\subsection{Product of Experts}

As presented, the joint variational posterior $q_\phi(\vz|\vx,\vy)$ could be a separate neural network with independent parameters with respect to $q_\phi(\vz|\vx)$ and $q_\phi(\vz|\vy)$. While this is no doubtedly expressive, it does not scale well to applications with more than two modalities: we would need to define $2^K$ inference networks for $K$ modalities, $(\vx_1, ..., \vx_K)$, which can quickly make learning infeasible. Previous work \citep{wu2018multimodal,vedantam2017generative} posed an elegant solution to this problem of scalability: define $q_\phi(\vz|\vx,\vy)$ as a product of the unimodal posteriors. Precisely, they use a product-of-experts \citep{hinton1999products}, also called PoE, to relate the true joint posterior to the true unimodal posteriors:
\begin{equation}
  p(\vz|\vx_1, ..., \vx_K) \propto \frac{\prod_{i=1}^K p(\vz|\vx_i)}{\prod_{i=1}^{K-1} p(\vz)} \approx \frac{\prod_{i=1}^K q(\vz|\vx_i)p(\vz)}{\prod_{i=1}^{K-1} p(\vz)} = p(\vz)\prod_{i=1}^K q(\vz|\vx_i)
  \label{eqn:poe:def}
\end{equation}
where each true posterior $p(\vz|\vx_i)$ is approximated by the product distribution $q(\vz|\vx_i)p(\vz)$, resulting in the product of $K$ experts and a ``prior expert" $p(\vz)$.

Given our new multimodal objective function, $\mathcal{L}_{\text{M}}$, we would like to preserve using the PoE for efficient inference. However, it remains to show that under certain assumptions, using the PoE in defining a joint posteriors will converge to the true distribution $p_\theta(\vz|\vx,\vy)$.

\begin{thm}
Let $p$ denote the true generative model and $q_\phi$ denote a variational inference network. If we assume a product-of-experts rule such that $q_\phi(\vz|\vx,\vy) = p(\vz)q_\phi(\vz|\vx)q_\phi(\vz|\vy)$ and if $q_\phi(\vz|\vx,\vy) = p(\vz|\vx,\vy)$ i.e. the true posterior is properly contained in the variational family, then the loss $\mathcal{L}_{\text{M}} = 0$.
\label{thm:4}
\end{thm}
\begin{proof}
To be precise, we make the following assumpmtions:
\begin{align}
    q(\vx,\vy,\vz) &= p_d(\vx,\vy)q_\phi(\vz|\vx,\vy)\label{ass:1} \\
    q(\vz|\vx,\vy) &= p(\vz|\vx,\vy) = \frac{p(\vz|\vx)p(\vz|\vy)}{p(\vz)}\cdot\frac{p(\vx)p(\vy)}{p(\vx,\vy)} \propto \frac{p(\vz|\vx)p(\vz|\vy)}{p(\vz)}\label{ass:2} \\
    p(\vy|\vz,\vx) &= p(\vy|\vz)\label{ass:3} \\
\end{align}
where $p_d(\vx,\vy)$ is an empirical distribution. We first observe that $p_d(\cdot) = p(\cdot)$ since if $p$ is the true generative model, then any empirical dataset would be sampled from $p$. Then,
\begin{align}
    \mathcal{L}_{\text{M}} &= \int_{\mathcal{X}} \int_{\mathcal{Y}} \int_{\mathcal{Z}} q(\vx,\vy,\vz) \log \frac{q(\vx,\vz)p(\vy|\vx,\vz)}{q(\vx,\vy,\vz)} d\vz d\vy d\vx\label{proofline:1} \\
    &= \int_{\mathcal{X}} \int_{\mathcal{Y}} p_d(\vx,\vy) \int_{\mathcal{Z}} q(\vz|\vx,\vy) \log \{ \frac{q(\vz|\vx)p(\vy|\vz)}{q(\vz|\vx,\vy)} \cdot \frac{p_d(\vx)}{p_d(\vx,\vy)} \} d\vz d\vy d\vx\label{proofline:2} \\
    &= \int_{\mathcal{X}} \int_{\mathcal{Y}} p_d(\vx,\vy) \int_{\mathcal{Z}} p(\vz|\vx,\vy) \log \{ \frac{p(\vz|\vx)p(\vy|\vz)}{p(\vz|\vx,\vy)} \cdot \frac{p_d(\vx)}{p_d(\vx,\vy)} \} d\vz d\vy d\vx\label{proofline:3} \\
    &= \int_{\mathcal{X}} \int_{\mathcal{Y}} p_d(\vx,\vy) \int_{\mathcal{Z}} p(\vz|\vx,\vy) \log \{ \frac{p(\vz|\vx)p(\vy|\vz)}{\frac{p(\vz|\vx)p(\vz|\vy)}{p(\vz)}\cdot\frac{p(\vx)p(\vy)}{p(\vx,\vy)}} \cdot \frac{p_d(\vx)}{p_d(\vx,\vy)} \} d\vz d\vy d\vx\label{proofline:4} \\
    &= \int_{\mathcal{X}} \int_{\mathcal{Y}} p_d(\vx,\vy) \int_{\mathcal{Z}} p(\vz|\vx,\vy) \log \{ \frac{p(\vy|\vz)p(\vz)p(\vx,\vy)}{p(\vz|\vy)p(\vx)p(\vy)} \cdot \frac{p_d(\vx)}{p_d(\vx,\vy)} \} d\vz d\vy d\vx\label{proofline:5} \\
    &= \int_{\mathcal{X}} \int_{\mathcal{Y}} p_d(\vx,\vy) \int_{\mathcal{Z}} p(\vz|\vx,\vy) \log \{ \frac{p(\vx,\vy)}{p(\vx)} \cdot \frac{p_d(\vx)}{p_d(\vx,\vy)} \} d\vz d\vy d\vx\label{proofline:6} \\
    &= \int_{\mathcal{X}} \int_{\mathcal{Y}} \int_{\mathcal{Z}} p(\vx,\vy,\vz) \log\{1\} d\vz d\vy d\vx\label{proofline:7} \\
    &= 0
\end{align}
We apply Assumption~\ref{ass:1},\ref{ass:3} in line~\ref{proofline:2}. We apply Assumption~\ref{ass:2} in line~\ref{proofline:3}. In line~\ref{proofline:6}, we use the fact $p(\vy|\vz)p(\vz) = p(\vz|\vy)p(\vy)$. In line~\ref{proofline:7}, we use that $p_d(\cdot) = p(\cdot)$.
\end{proof}

Being able to use PoE in our new family of VAEs can be especially important in cases of missing data. We will compare the performance of PoE against specifying a new neural network for $q_\phi(\vz|\vx,\vy)$ in our experiments.

\section{VAE Experiments}
\label{sec:vae_experiments}
To benchmark our proposed family of models, we first compare against existing multimodal VAEs and supervised baselines on a variety of bimodal datasets involving images and labels. Then, we tackle the challenge of images and language. See supplement for a description of each dataset, the model architectures, and hyperparameters.

\subsection{Images and Labels}
We transform unimodal image datasets into multimodal ones by treating labels as a second modality ($\vx$ represents images, $\vy$ labels): MNIST, FashionMNIST, CelebA, and CIFAR10. All images are resized to 32 by 32 pixels and all models have identical architectures. For evaluation, we compute the marginal and joint log likelihoods on an unseen test set estimated using 1000 samples (see supplement).
In lieu for the conditional probability $p_\theta(\vy|\vx)$, we compute the classification error of predicting the correct label from an image. We do so by taking $\mu$, the mean of the posterior $q_\phi(\vz|\vx)$, and decoding to the MAP label from $p_\theta(\vy|\vz=\mathbf{\mu})$.
Table~\ref{table:image_label} reports evaluation metrics averaged over 5 runs.

\begin{table*}[t!]
\centering
\footnotesize
\begin{subtable}[t]{\textwidth}
\centering
\begin{tabular}{l|ccc|ccc}
    \toprule
    & \multicolumn{3}{c}{MNIST} & \multicolumn{3}{c}{FashionMNIST}\\
    Model & $\log p(\vx)$ & $\log p(\vx,\vy)$ & Cls. Err. & $\log p(\vx)$ & $\log p(\vx,\vy)$ & Cls. Err. \\
    \midrule
    VAE & $-113.26$ & - & - & $-299.70$ & - & - \\
    Classifier & - & - & $0.984$ & - & - & $0.908$ \\
    BiVCCA & $-114.15$ & - & $\mathbf{0.989}$ & $-303.93$ & - & $\mathbf{0.915}$ \\
    MVAE & $-112.61$ & $-113.62$ & $0.884$ & $\mathbf{-299.19}$ & $-299.91$ & $0.787$ \\
    JMVAE & $-113.81$ & $-113.48$ &  $0.983$ & $-301.47$ & $-299.91$ & $0.900$ \\
    VAEVAE & $\mathbf{-112.44}$ & $\mathbf{-112.65}$ & $0.985$ & $-299.21$ & $\mathbf{-299.34}$ & $0.910$ \\
    VAEVAE$^\ddagger$ & $-112.78$ & $-113.33$ & $0.983$ & $-299.89$ & $-299.64$ & $0.906$ \\
    \bottomrule
\end{tabular}
\end{subtable}
\newline\vspace{2em}\newline
\begin{subtable}[t]{\textwidth}
\centering
\begin{tabular}{l|ccc|ccc}
    \toprule
    & \multicolumn{3}{c}{CIFAR10} & \multicolumn{3}{c}{CelebA}\\
    Model & $\log p(\vx)$ & $\log p(\vx,\vy)$ & Cls. Err. & $\log p(\vx)$ & $\log p(\vx,\vy)$ & Cls. Err. \\
    \midrule
    VAE $\mathbf{-1774.86}$ & - & - & $\mathbf{-1582.69}$ & - & - \\
    Classifier & - & - & $0.637$ & - & - & $\mathbf{0.923}$ \\
    BiVCCA & $-1785.33$ & - & $\mathbf{0.654}$ & $-1596.08$ & - & $0.922$ \\
    MVAE & $-1775.33$ & $-1778.80$ & $0.511$ & $-1584.25$ & $-1588.77$ & $0.841$ \\
    JMVAE & $-1780.82$ & $-1778.27$ & $0.541$ & $-1595.13$ & $-1588.76$ & $0.914$ \\
    VAEVAE & $-1775.63$ & $\mathbf{-1777.25}$ & $0.607$ & $-1584.31$ & $\mathbf{-1586.92}$ & $0.920$ \\
    VAEVAE$^\ddagger$ & $-1776.71$ & $-1778.45$ & $0.601$ & $-1586.95$ & $-1587.20$ & $0.918$ \\
    \bottomrule
\end{tabular}
\end{subtable}
\caption{Evaluation of Multimodal VAEs on Image and Label Datasets. We also include a supervised baseline (Classifier) and an image-only generative model (VAE). The model VAEVAE$^\ddagger$ uses Product-of-Experts in the variational posterior. All models share the same architecture and hyperparameters. To evaluate each of the models, we compare log probabilities of test data and approximate conditional distributions by classification error (Cls. Err.) on test data.}
\label{table:image_label}
\end{table*}
\begin{table*}[t!]
\centering
\footnotesize
\begin{subtable}[t]{\textwidth}
\centering
\begin{tabular}{l|cccc|cccc}
    \toprule
    & \multicolumn{4}{c}{MNIST Math} & \multicolumn{4}{c}{Chairs in Context}\\
    Model & $\log p(\vx)$ & $\log p(\vy)$ & $\log p(\vx,\vy)$ & Perp. & $\log p(\vx)$ & $\log p(\vy)$ & $\log p(\vx,\vy)$ & Perp. \\
    \midrule
    VAE ($\vx$) & $-113.21$ & - & - & - & $\mathbf{-592.6}$ & - & - & - \\
    VAE ($\vy$) & - & $-0.143$ & - & - & - & $-10.8$ & - & - \\
    Captioner & - & - & - & $2.231$ & - & - & - & $\mathbf{1.7e3}$ \\
    BiVCCA & $-114.81$ & $-0.193$ & - & $\mathbf{2.196}$ & $-601.2$ & $-10.9$ & - & $2.5e4$ \\
    MVAE & $\mathbf{-112.62}$ & $-0.141$ & $-113.72$ & $1.6e4$ & $-607.9$ & $-10.8$ & $-636.2$ & $1.5e6$ \\
    JMVAE & $-115.73$ & $-0.410$ & $-114.00$ & $9.5e3$ & $-638.0$ & $-13.3$ & $-619.7$ & $4.8e6$ \\
    VAEVAE & $-112.88$ & $\mathbf{-0.133}$ & $\mathbf{-113.03}$ & $745.2$ & $-602.3$ & $-10.5$ & $\mathbf{-619.5}$ & $1.7e5$ \\
    VAEVAE$^\ddagger$ & $-113.41$ & $-0.133$ & $-113.69$ & $274.0$ & $-602.3$ & $\mathbf{-10.4}$ & $-620.6$ & $8.2e5$\\
    \bottomrule
\end{tabular}
\end{subtable}
\newline\vspace{2em}\newline
\begin{subtable}[t]{\textwidth}
\centering
\begin{tabular}{l|cccc|cccc}
    \toprule
    & \multicolumn{4}{c}{Flickr30k} & \multicolumn{4}{c}{COCO Captions}\\
    Model & $\log p(\vx)$ & $\log p(\vy)$ & $\log p(\vx,\vy)$ & Perp. & $\log p(\vx)$ & $\log p(\vy)$ & $\log p(\vx,\vy)$ & Perp. \\
    \midrule
    VAE ($\vx$) & $-1772.8$ & - & - & - & $\mathbf{-1769.5}$ & - & - & - \\
    VAE ($\vy$) & - & $-65.1$ & - & - & - & $-14.3$ & - & - \\
    Captioner & - & - & - & $170.4$ & - & - & - & $\mathbf{121.2}$ \\
    BiVCCA & $-1883.8$ & $-65.4$ & - & $\mathbf{153.9}$ & $-1778.4$ & $-14.9$ & - & $133.8$ \\
    MVAE & $\mathbf{-1781.0}$ & $-65.8$ & $-1943.3$ & $4.7e5$ & $-1778.4$ & $-14.5$ & $-1875.9$ & $3.6e6$ \\
    JMVAE & $-1967.2$ & $-75.9$ & $-1872.9$ & $1.6e4$ & $-2028.8$ & $-21.7$ & $-1823.3$ & $4.3e9$ \\
    VAEVAE & $-1804.8$ & $\mathbf{-64.4}$ & $\mathbf{-1825.1}$ & $7.7e3$ & $-1786.8$ & $\mathbf{-14.1}$ & $\mathbf{-1818.7}$ & $3.2e4$ \\
    VAEVAE$^\ddagger$ & $-1814.2$ & $-65.2$ & $-1850.2$ & $7.6e3$ & $-1806.4$ & $-14.1$ & $-1834.5$ & $6.5e4$\\
    \bottomrule
\end{tabular}
\end{subtable}
\caption{Evaluation of Multimodal VAEs on Image and Text Datasets. We include a supervised captioning model (Captioner) and two unimodal generative models (VAE ($\mathbf{x}$), VAE ($\mathbf{y}$)). The model VAEVAE$^\ddagger$ uses Product-of-Experts in the variational posterior. All models share the same architecture and hyperparameters. To evaluate each of the models, we compare log probabilities of test data and approximate conditional distributions by the  perplexity of captions conditioned on the image (perp.) using test data.}
\label{table:image_text}
\end{table*}

\paragraph{Results} Previous multimodal models either bias towards capturing marginal or conditional distributions. In fact, many of their objective functions fail to lower bound one or the other (see Sec.~\ref{sec:related_work}). In Table~\ref{table:image_label}, we find that BiVCCA and JMVAE have lower (more negative) marginal and joint log-likelihoods than other models, only doing well in classification error (conditionals). On the other hand, MVAE is the opposite, resulting in poor classification and overfitting to the marginals. However, our class of generative models (VAEVAE), which has a principled lower bound, finds a good balance between learning marginals versus conditionals. In all four datasets, VAEVAE does better than BiVCCA and JMVAE in capturing $\log p(\vx)$ and $\log p(\vx,\vy)$, almost matching the performance of a unimodal VAE. Furthermore, VAEVAE outperforms MVAE in classification, nearing supervised models. Finally, we find that VAEVAE with PoE, denoted by VAEVAE$^{\ddagger}$, closely approximates the standard VAEVAE, suggesting that we can trade a little performance for the efficiency that comes with PoE.

\subsection{Images and Text}
Many prior work in multimodal generative models \citep{wang2016deep,suzuki2016joint,vedantam2017generative,wu2018multimodal} only testing their methods against simple images and labels (as above). However, we would expect many more benefits to multimodal learning given two rich modalities. So in these next set of experiments, we apply our method to four datasets with images ($\vx$) and language ($\vy$). The first is a toy setting based on MNIST where we procedurally generate math equations that evaluate to a digit in the dataset. The second is the ``Chairs In Context`` (CIC) dataset \citep{achlioptas2018learning}, which contains 80k human utterances paired with images of chairs from ShapeNet \citep{chang2015shapenet}. The final two are Flickr30k \citep{plummer2015flickr30k} and COCO captions \citep{chen2015microsoft}, which contain richer images and text. Like above, we evaluate models with estimates of marginal $\log p_\theta(\vx)$, $\log p_\theta(\vy)$ and joint $\log p_\theta(\vx,\vy)$ log probabilities with respect to a test set. Instead of classification, we measure the (conditional) perplexity of $\vy$ given $\vx$ as a surrogate for $\log p_\theta(\vy|\vx)$, which is difficult to directly estimate. A better generative model should have lower perplexity, or less surprisal for every word in a sentence.
To encode and decode utterances, we use RNNVAEs \citep{bowman2015generating}. Table~\ref{table:image_text} compare model performances. For baselines, we include a supervised captioning model and two VAEs, one exclusively for images and one exclusively for text. Once again, all models share  architectures and hyperparameters, meaning that any differences in metrics are due to the objectives alone.

\paragraph{Results} We find similar patterns as described for image and label datasets, although the magnitude of the effect is higher for image and text. As before, JMVAE has poor performance in marginal and joint probabilities whereas MVAE has poor (conditional) perplexity. Our VAEVAE models often matches the MVAE in $\log p_\theta(\vx)$, $\log p_\theta(\vy)$, and $\log p_\theta(\vx,\vy)$ and surpasses JMVAE in perplexity. From Table~\ref{table:image_text}, we can see that supervised approaches like a captioner or BiVCCA tend to do far better in perplexity due to heavily overfitting to learning only the mapping from image to text, a much lower dimensional transformation than learning a full joint representation. But out of the joint generative models, VAEVAE outperforms others by orders of magnitude in perplexity. If a practitioner only cared about image captioning, he/she should use a supervised model. But as we will explore in the next section, in cases of missing data or low supervision, discriminative approaches quickly deteriorate. We will show how VAEVAE is able to leverage unpaired marginal data to still learn a good joint representation under data scarcity.

\section{Weak Supervision}
\label{sec:weak_sup}
In many contexts, it is infeasible to curate a large dataset of paired modalities (e.g. matching image and text), especially as the number of modalities increase. Often, we only have the resources to build a very small set of paired examples. However, unimodal (or unpaired) data is easily accessible both from existing corpora like ImageNet \citep{russakovsky2015imagenet} or the Google Books corpus \citep{michel2011quantitative} and from scraping online content from Wikipedia or Google Images. One of the strengths of some generative models (in comparison to a supervised one) is the ability to do inference and make predictions despite missing data. As such, we might wish to use unpaired examples to thoroughly capture marginal distributions whereas existing paired examples can be used to align unimodal representations into a joint space.
For the VAEVAE, since $\mathcal{L}_{M}$ is composed of two marginal and two conditional terms, we must use paired examples to evaluate the latter but are free to use unpaired examples to evaluate the former. \cite{wu2018multimodal} began to investigate using unpaired examples to training the MVAE. Here, we continue with more thorough experiments where we reserve a fraction of a dataset to be paired and randomly scramble the remaining examples to remove alignment. We vary the fraction of paired examples to gauge how increased amounts of ``supervision" effects training. Then, using a fixed amount of supervision, we vary the amount of unpaired examples to test if learning marginal distributions better can make joint representation learning easier.

\subsection{Varying Amount of Paired Data}
First we vary the amount of paired examples from 0.1\% to 100\% of the training dataset size. For each setting, we track classification accuracy for labelled datasets and conditional perplexity for image and text datasets. Only VAEVAE and MVAE can make use of unpaired data. All other baseline models only train on the supervised fraction.

Fig.~\ref{fig:weaksup1}(a-d) plot the classification accuracy against log of the supervision percentage (higher is better). We find VAEVAE to consistently do better in the low supervision regime (-2 to 0 log supervision). However as we approach full supervision, the supervised baseline and BiVCCA tend to outperform VAEVAE. Given enough paired data, this is expected as it is much easier to learn $p(\vy|\vx)$ alone rather the marginals, conditionals, and joint altogether. Fig.~\ref{fig:weaksup1}(e-h) show the marginal (log) densities over images in each dataset. Again we find VAEVAE to outperform most models, occasionally being beat by the MVAE.
Fig.~\ref{fig:weaksup2} show similar findings to Fig.~\ref{fig:weaksup1} but for image and text datasets. Fig.~\ref{fig:weaksup2}(a-d) plot the log perplexity of observed text conditioned on the image (lower is better) against log supervision percentage; Fig.~\ref{fig:weaksup2}(e-h) each contain two plots (the left representing the marginal log densities for image and the right for text). The patterns are the same: VAEVAE does the best at low supervision levels and captures the marginal distributions nicely. These patterns lay bare evidence that (1) MVAE over-prioritizes the marginals as it consistently achieves the lowest accuracy under low supervision, (2) JMVAE over-prioritizes the conditionals as it consistently gets the lowest marginal log likelihood, and (3) VAEVAE achieves a good balance between the two.

\subsection{Varying Amount of Unpaired Data}
Using more unpaired data should help models better capture marginal distributions. We hypothesize that this could make learning the joint and conditional distributions easier as well: Imagine that the joint representation is a ``simple" transformation from the unimodal representations; if we learn the latter well, any existing supervision can be used to learn that transformation. To test this intuition, we fix the amount of supervision to 1\% and vary the amount of unpaired data (for each modality) from 0.1\% to 100\%.  Again, only VAEVAE and MVAE can make use of unpaired data.

Fig.~\ref{fig:unsup1} show the effects on classification accuracy (a), conditional perplexity (b), and marginal log-likelihoods (c,d) as more unpaired data is given to the model. With little unpaired data, performance suffers greatly as generative models struggle to learn the underlying structure in the data distribution, leading to worse performance than a supervised baseline. But as the amount of unpaired data increases, classification error and conditional perplexity steadily decrease, beating the supervised baseline. Since unpaired data comes from the same distribution as paired data, more of it consistently leads to increased performance across datasets.


\section{Generalization to Other Model Families}

Thm.~\ref{thm:1} as stated, is not restricted to VAEs. In this section, we generalize the multimodal objective to GANs and flows.

\begin{figure*}[t]
  \centering
  \begin{subfigure}[h]{\textwidth}
      \centering
      \includegraphics[width=0.75\linewidth]{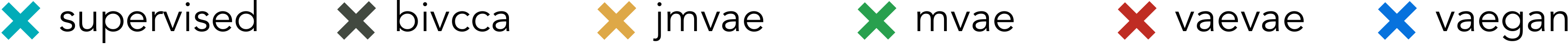}
  \end{subfigure}
  \begin{subfigure}[h]{0.24\textwidth}
      \centering
      \includegraphics[width=\linewidth]{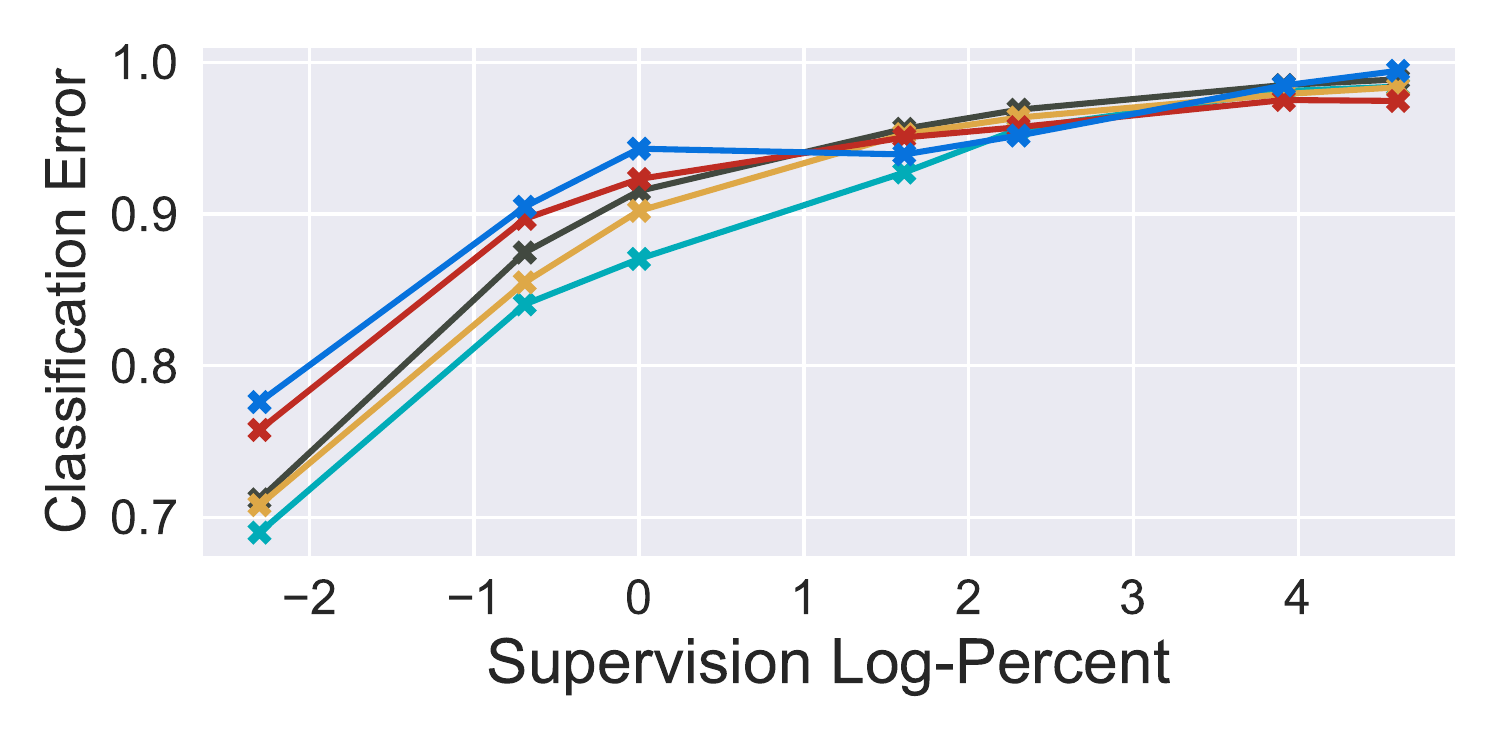}
      \caption{MNIST}
  \end{subfigure}
  \begin{subfigure}[h]{0.24\textwidth}
      \centering
      \includegraphics[width=\linewidth]{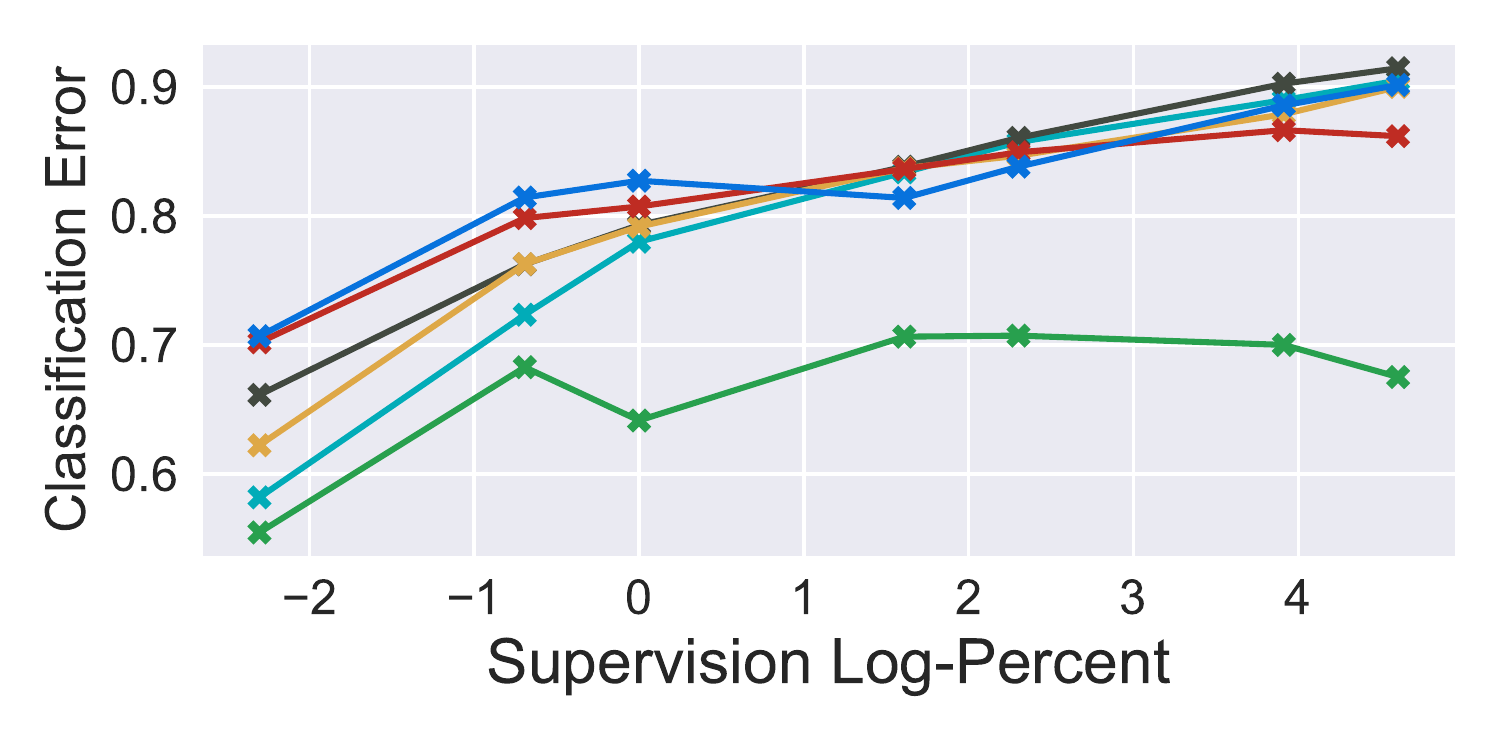}
      \caption{FashionMNIST}
  \end{subfigure}
  \begin{subfigure}[h]{0.24\textwidth}
      \centering
      \includegraphics[width=\linewidth]{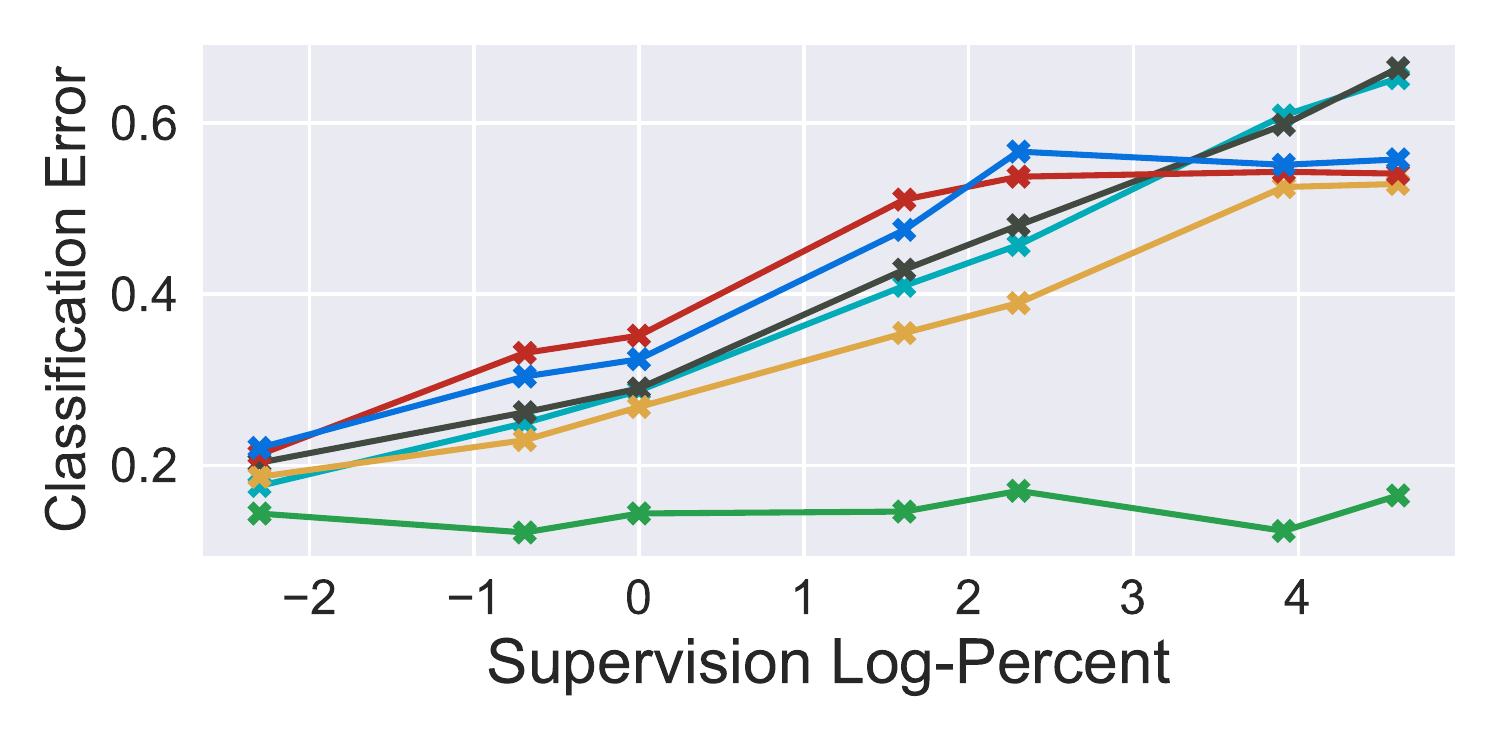}
      \caption{CIFAR10}
  \end{subfigure}
  \begin{subfigure}[h]{0.24\textwidth}
      \centering
      \includegraphics[width=\linewidth]{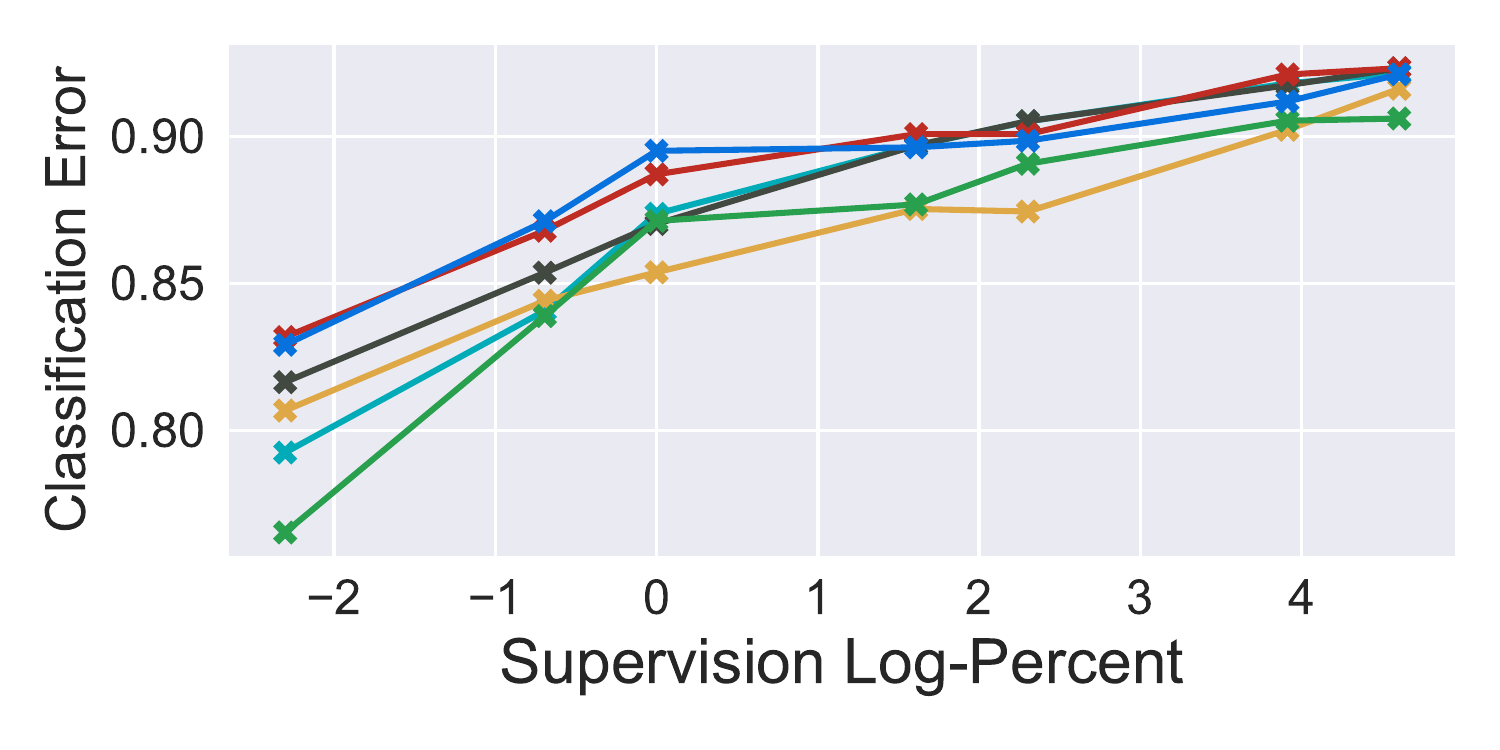}
      \caption{CelebA}
  \end{subfigure}
  \begin{subfigure}[h]{0.24\textwidth}
      \centering
      \includegraphics[width=\linewidth]{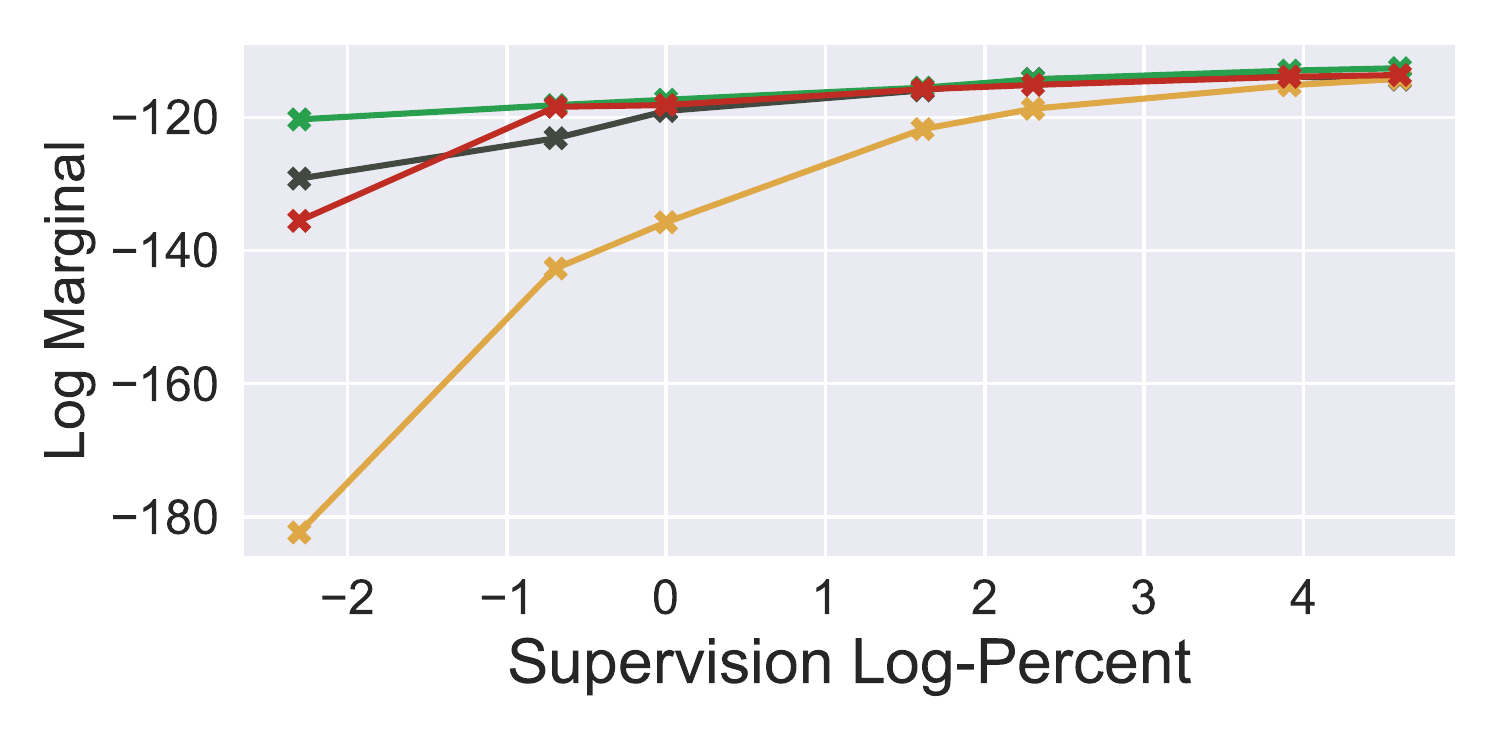}
      \caption{MNIST}
  \end{subfigure}
  \begin{subfigure}[h]{0.24\textwidth}
      \centering
      \includegraphics[width=\linewidth]{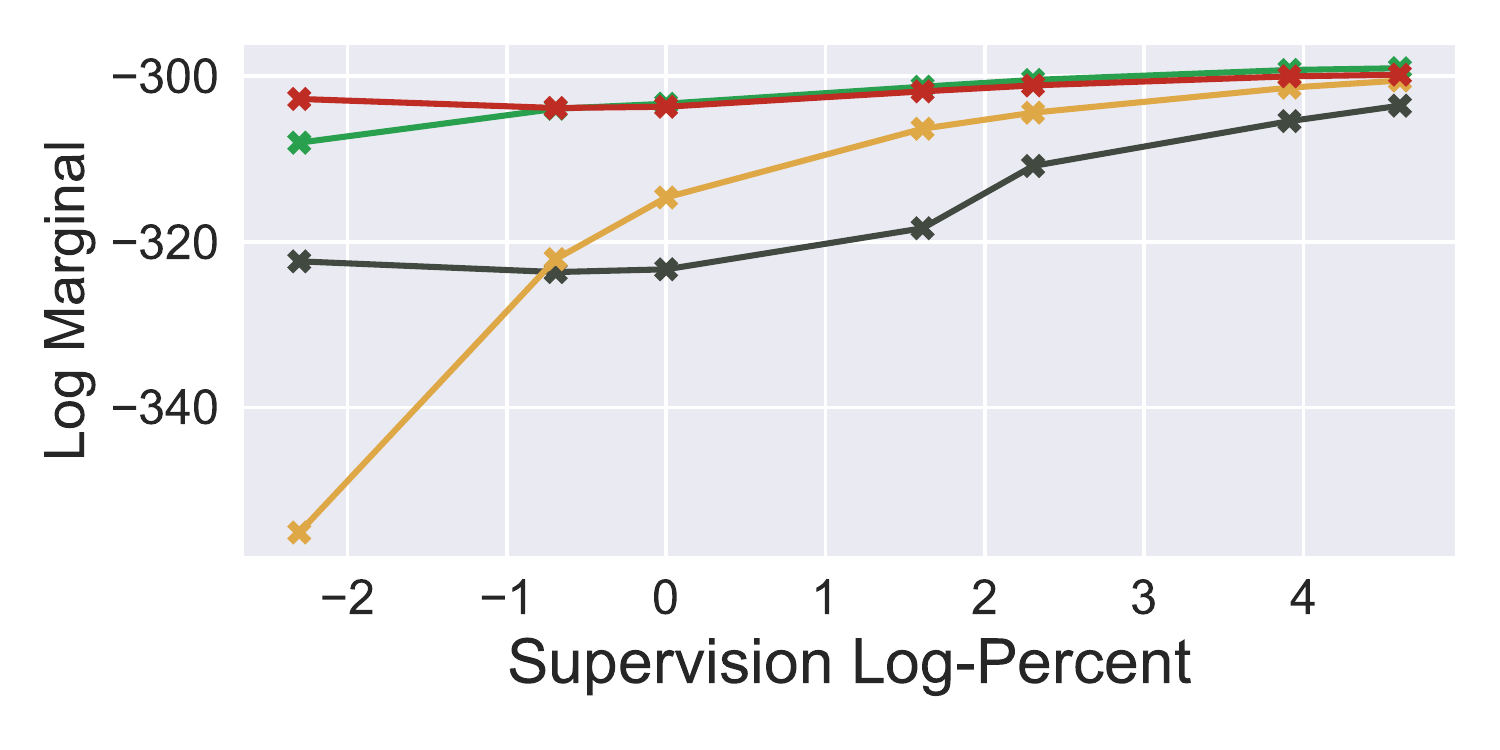}
      \caption{FashionMNIST}
  \end{subfigure}
  \begin{subfigure}[h]{0.24\textwidth}
      \centering
      \includegraphics[width=\linewidth]{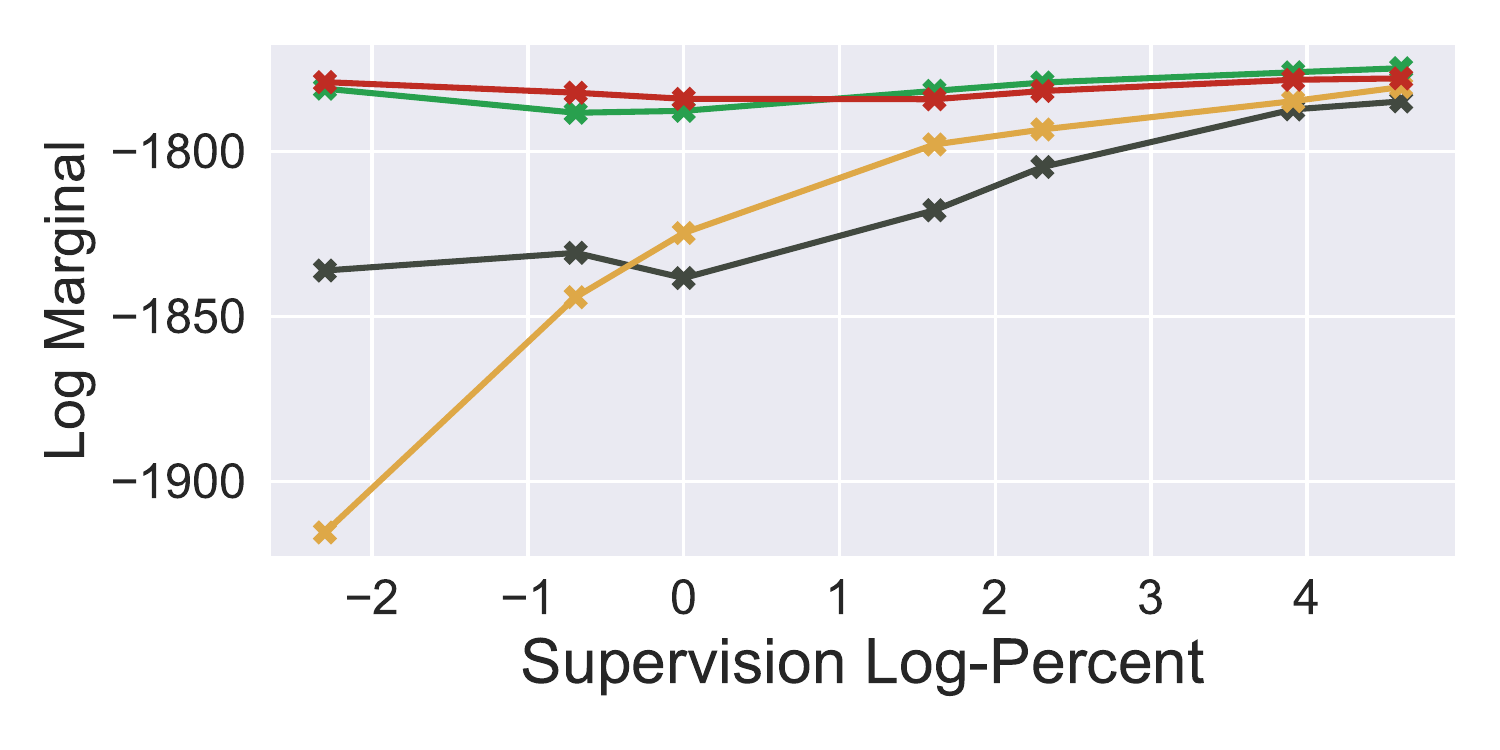}
      \caption{CIFAR10}
  \end{subfigure}
  \begin{subfigure}[h]{0.24\textwidth}
      \centering
      \includegraphics[width=\linewidth]{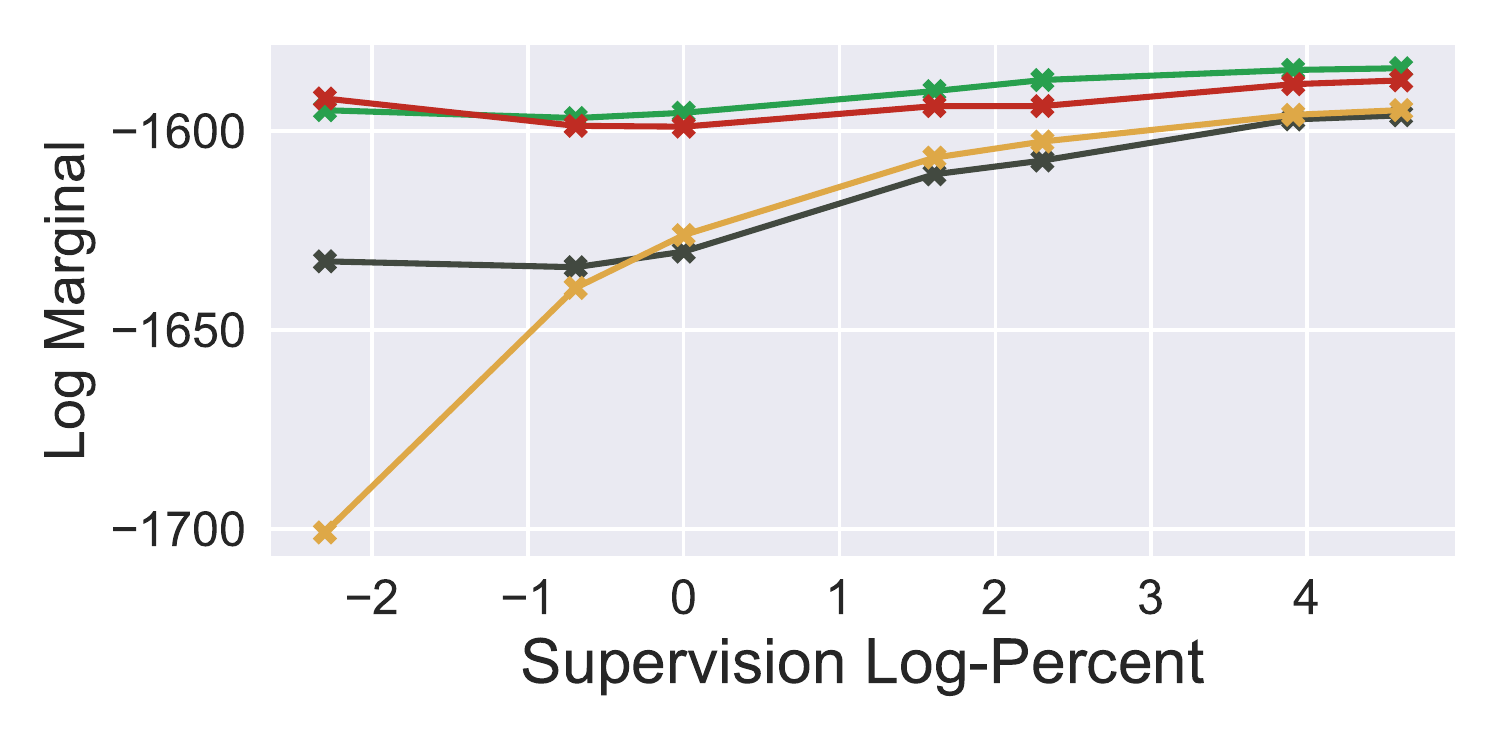}
      \caption{CelebA}
  \end{subfigure}
  \caption{Effect of weak supervision on classification (a-d) and marginal density (e-h). Our proposed models have higher accuracy in low to medium supervision. Being able to use unpaired examples ensures much better density estimation. All results are averaged over 3 runs.}
  \label{fig:weaksup1}
\end{figure*}

\begin{figure*}[t]
\centering
  \begin{subfigure}[h]{\textwidth}
      \centering
      \includegraphics[width=0.75\linewidth]{weaksup/legend.pdf}
  \end{subfigure}
  \begin{subfigure}[h]{0.24\textwidth}
      \centering
      \includegraphics[width=\linewidth]{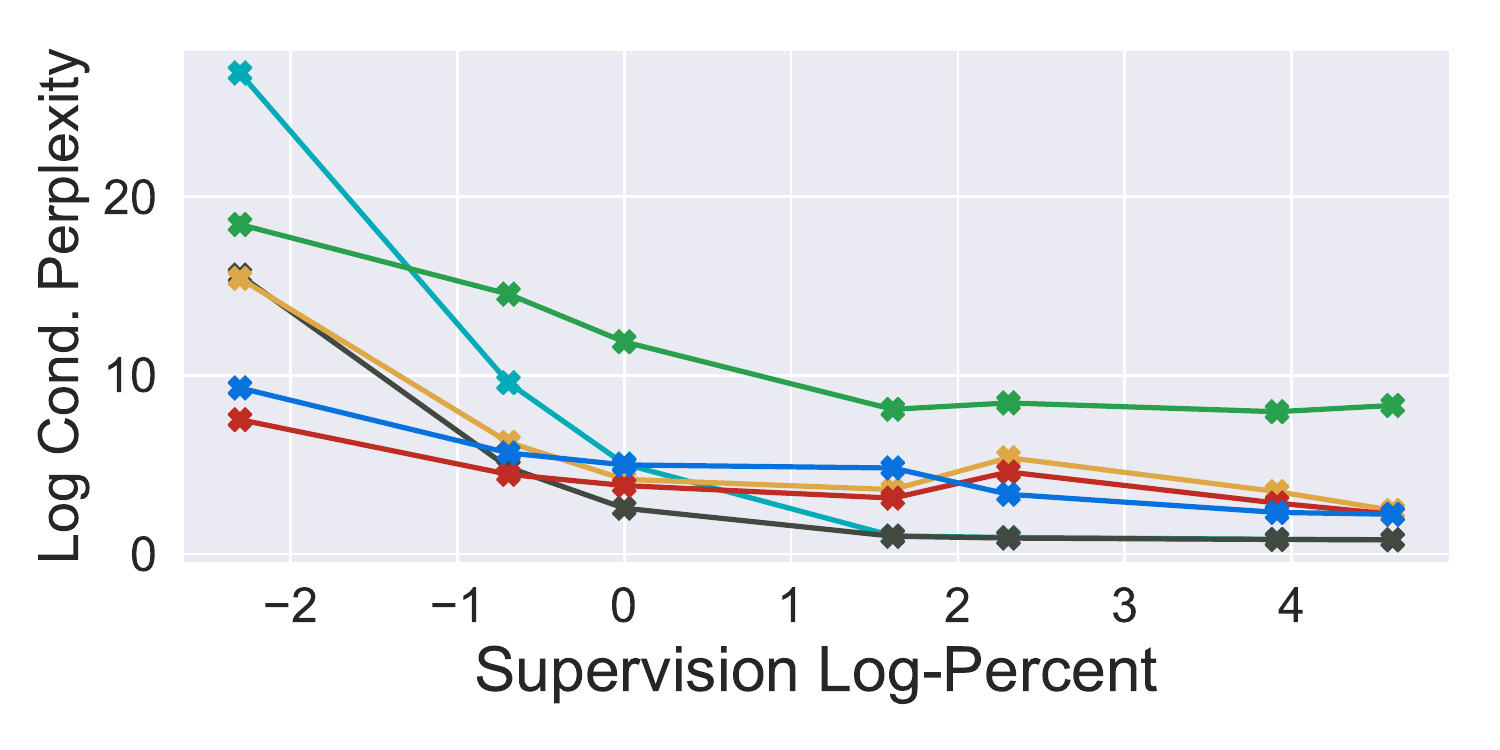}
      \caption{MNIST-Math}
  \end{subfigure}
  \begin{subfigure}[h]{0.24\textwidth}
      \centering
      \includegraphics[width=\linewidth]{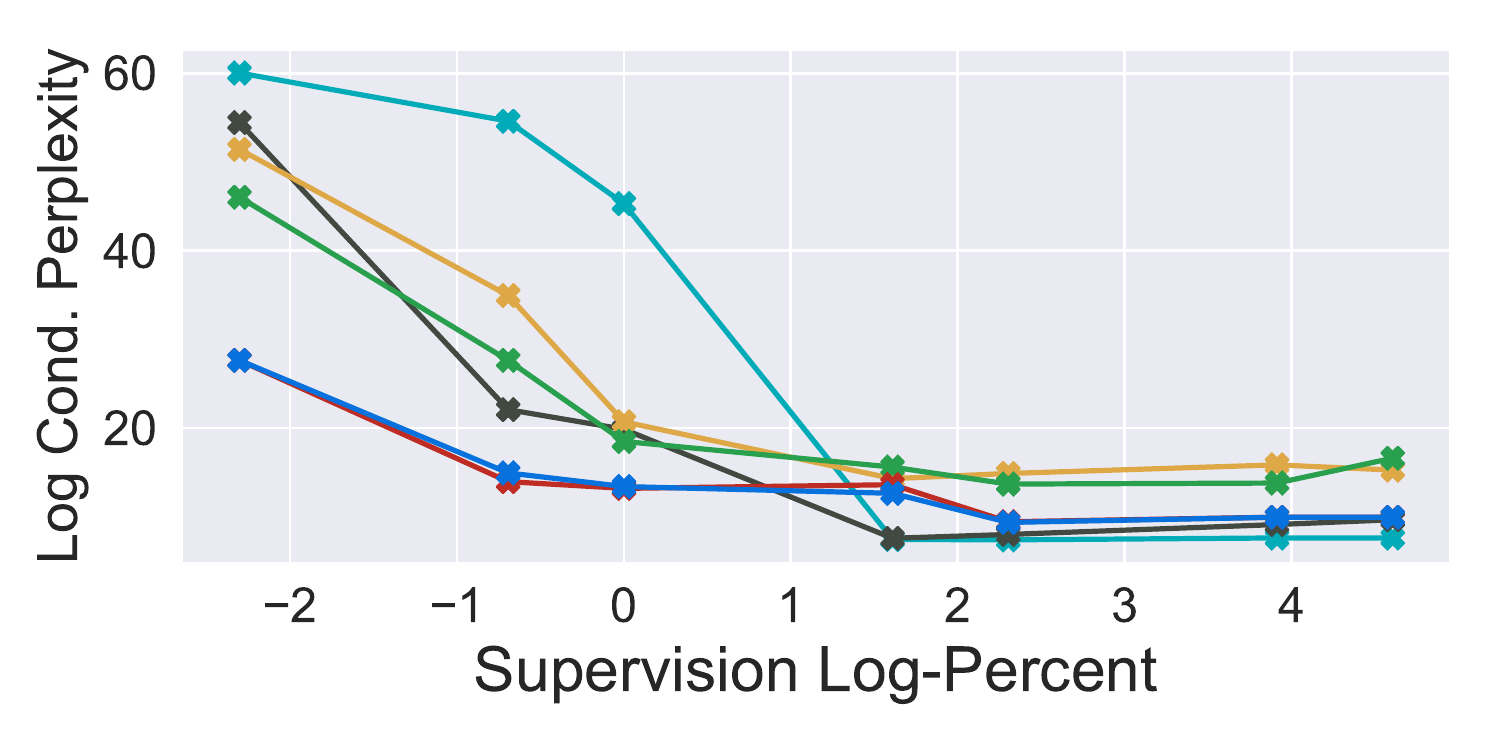}
      \caption{CIC}
  \end{subfigure}
  \begin{subfigure}[h]{0.24\textwidth}
      \centering
      \includegraphics[width=\linewidth]{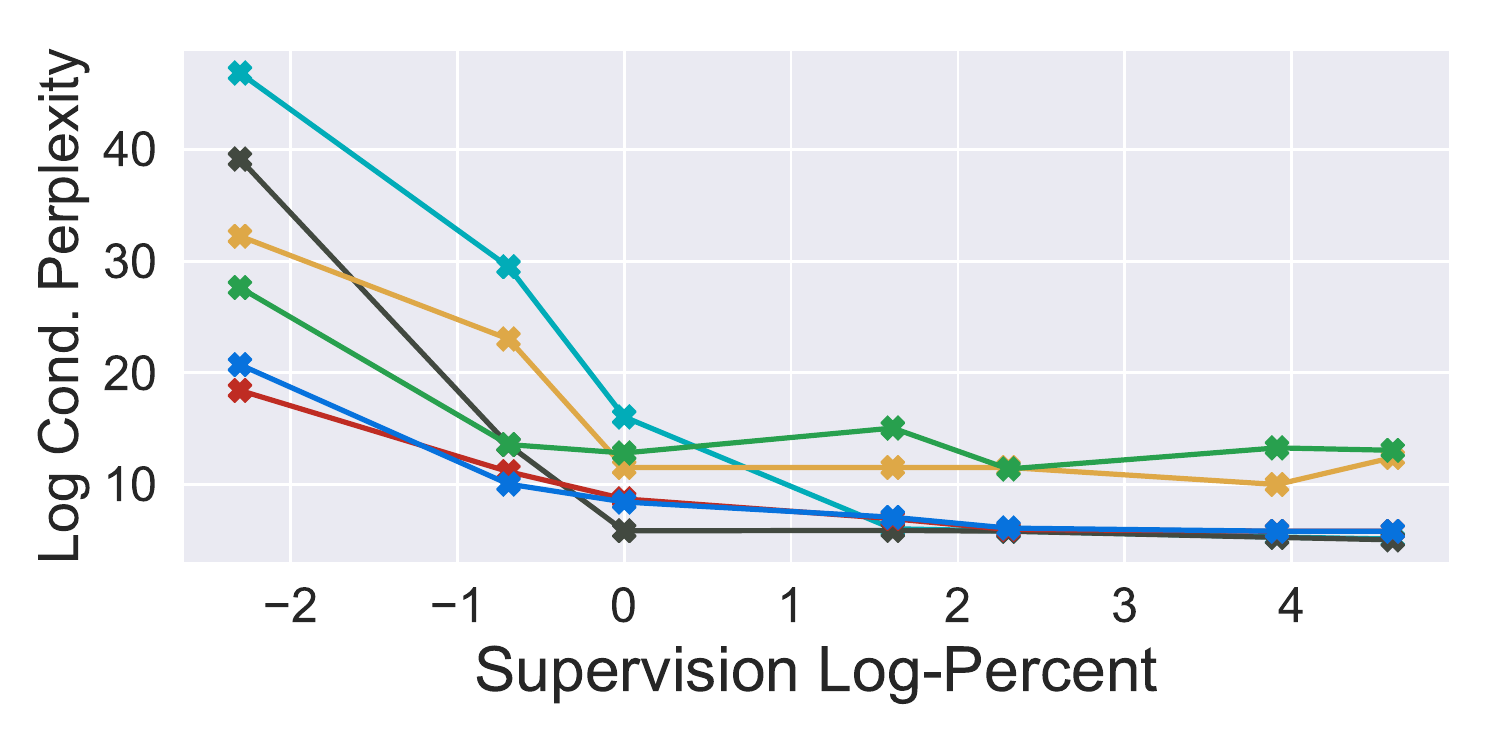}
      \caption{Flickr30k}
  \end{subfigure}
  \begin{subfigure}[h]{0.24\textwidth}
      \centering
      \includegraphics[width=\linewidth]{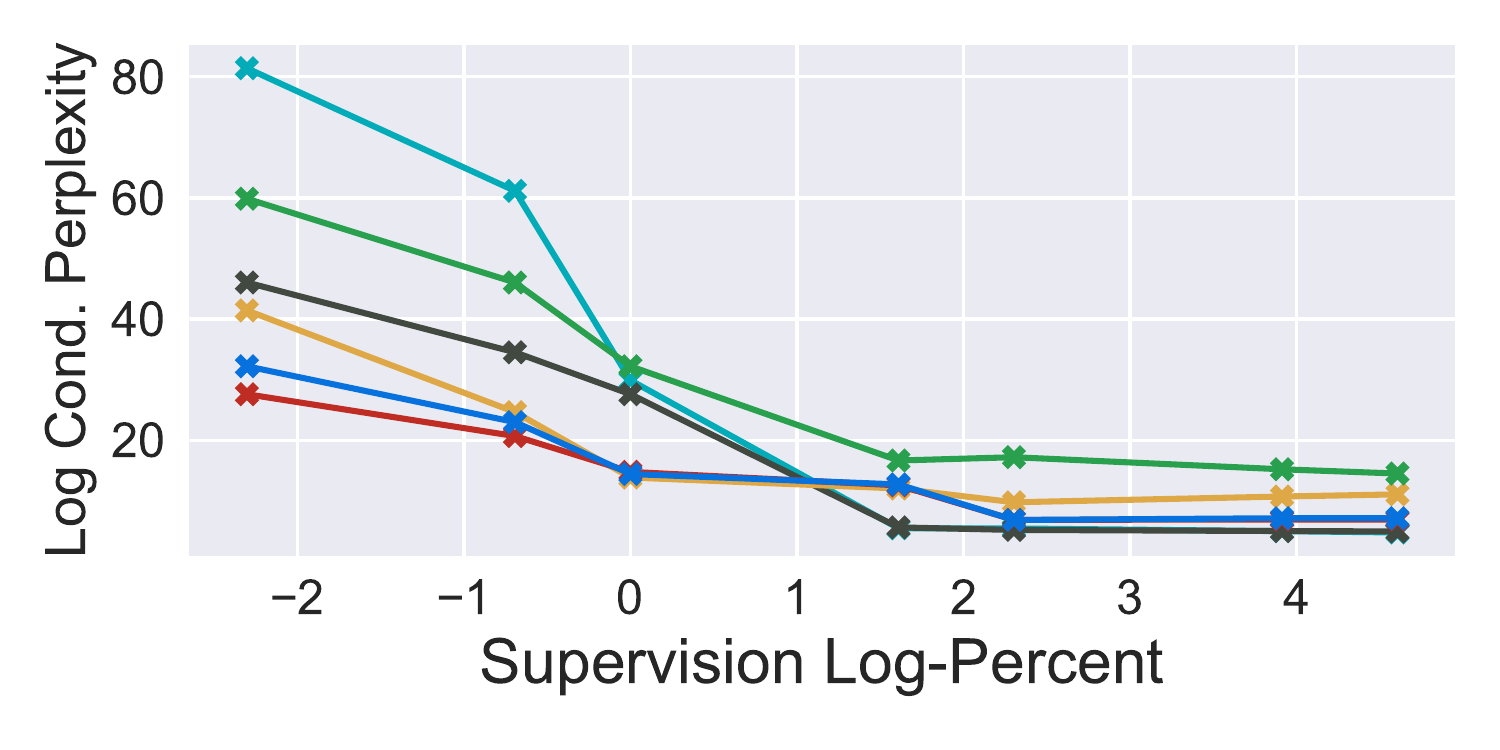}
      \caption{COCO}
  \end{subfigure}
  \centering
  \begin{subfigure}[h]{0.24\textwidth}
      \centering
      \includegraphics[width=\linewidth]{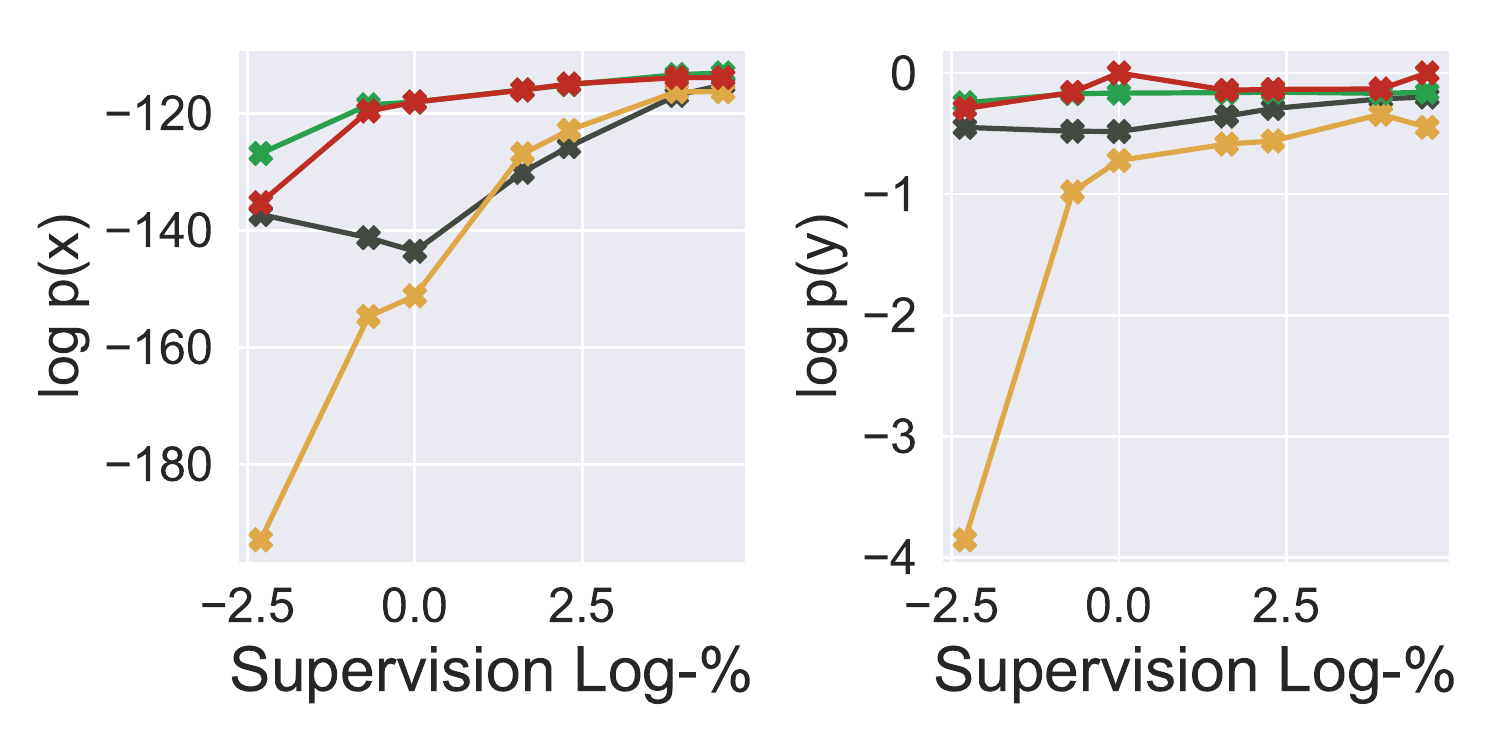}
      \caption{MNIST-Math}
  \end{subfigure}
  \begin{subfigure}[h]{0.24\textwidth}
      \centering
      \includegraphics[width=\linewidth]{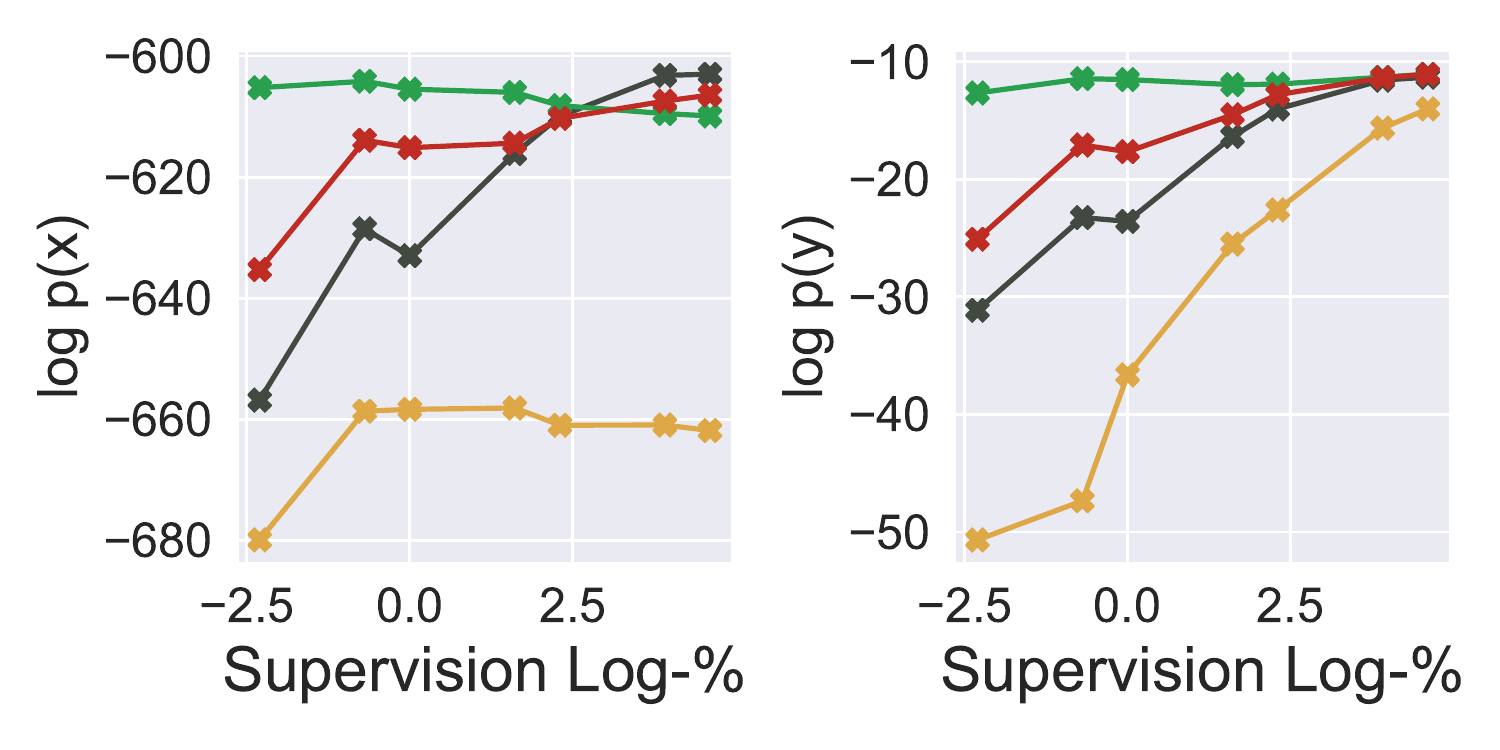}
      \caption{CIC}
  \end{subfigure}
  \begin{subfigure}[h]{0.24\textwidth}
      \centering
      \includegraphics[width=\linewidth]{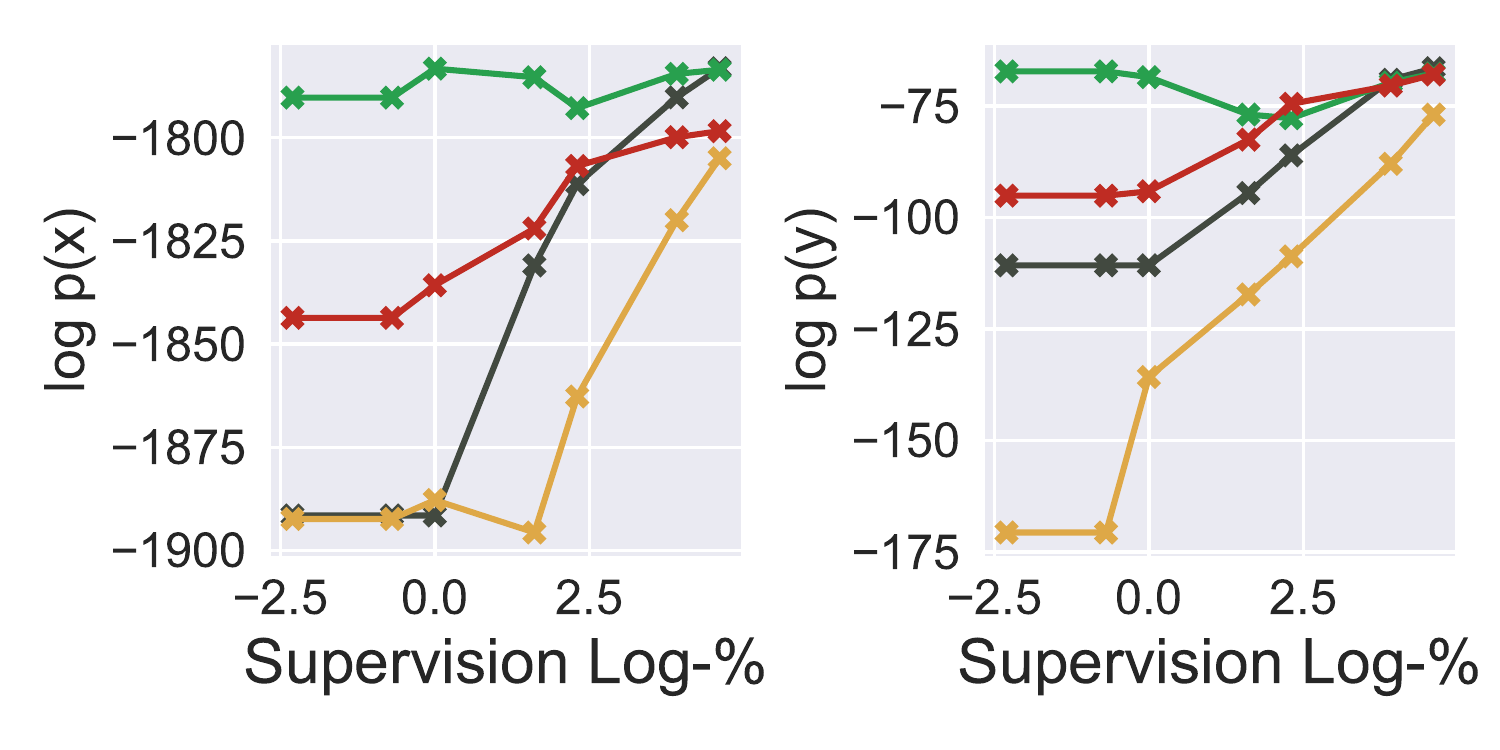}
      \caption{Flickr30k}
  \end{subfigure}
  \begin{subfigure}[h]{0.24\textwidth}
      \centering
      \includegraphics[width=\linewidth]{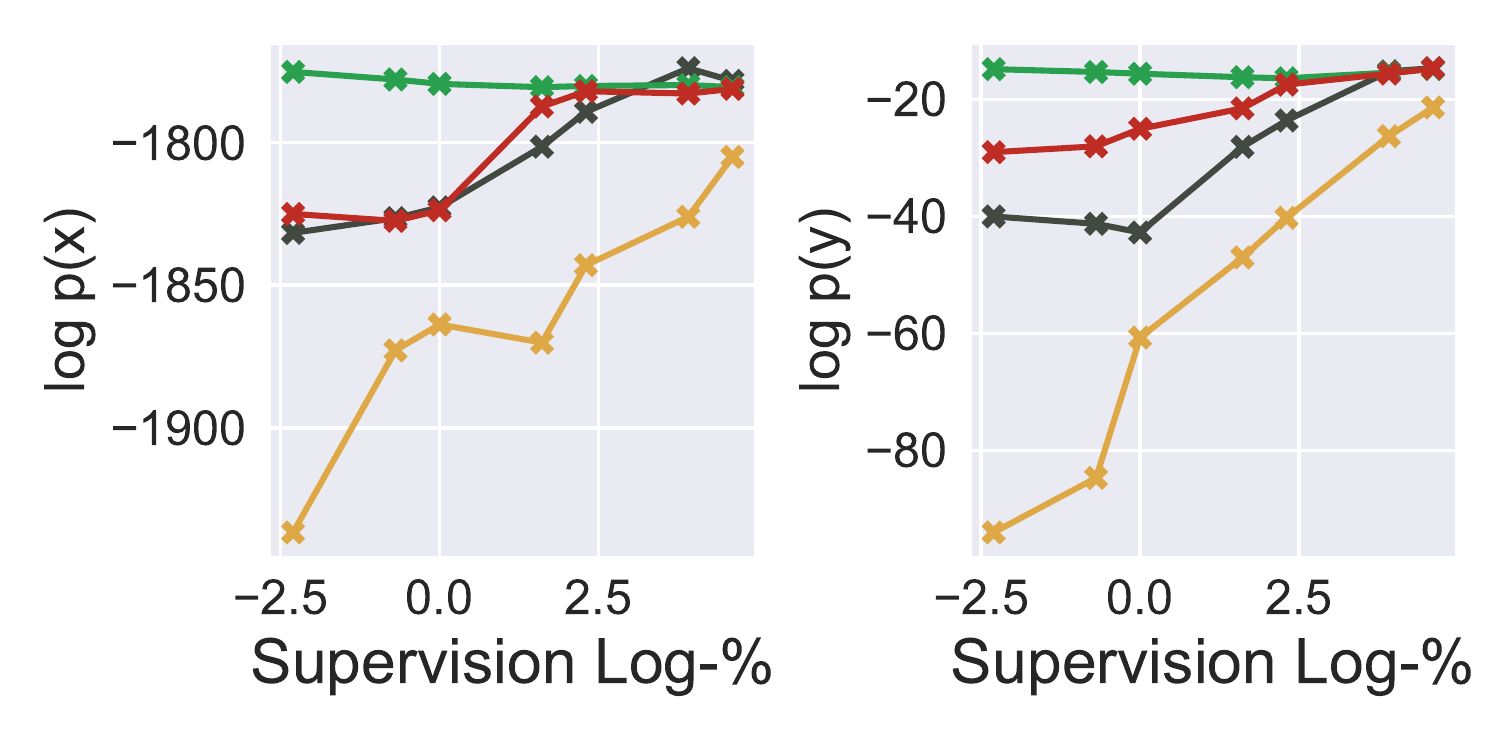}
      \caption{COCO}
  \end{subfigure}
  \caption{Effect of weak supervision on conditional perplexity (a-d) and marginal densities for image and text (e-h). VAEVAE uses unpaired data to learn conditional and marginal distributions, contributing to its strong performance.}
  \label{fig:weaksup2}
\end{figure*}
\begin{figure*}[t]
\centering
  \begin{subfigure}[h]{\textwidth}
      \centering
      \includegraphics[width=0.5\linewidth]{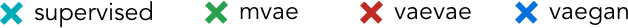}
  \end{subfigure}
  \begin{subfigure}[h]{0.24\textwidth}
      \centering
      \includegraphics[width=\linewidth]{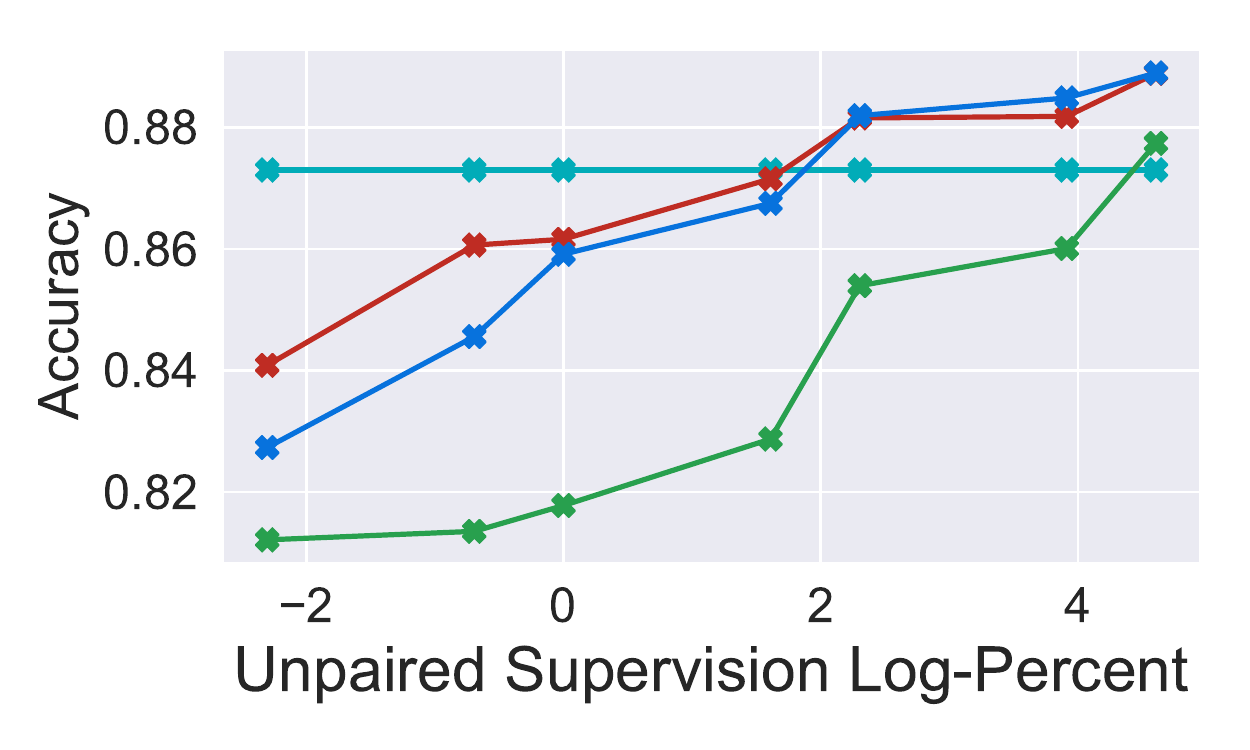}
      \caption{CelebA}
  \end{subfigure}
  \begin{subfigure}[h]{0.24\textwidth}
      \centering
      \includegraphics[width=\linewidth]{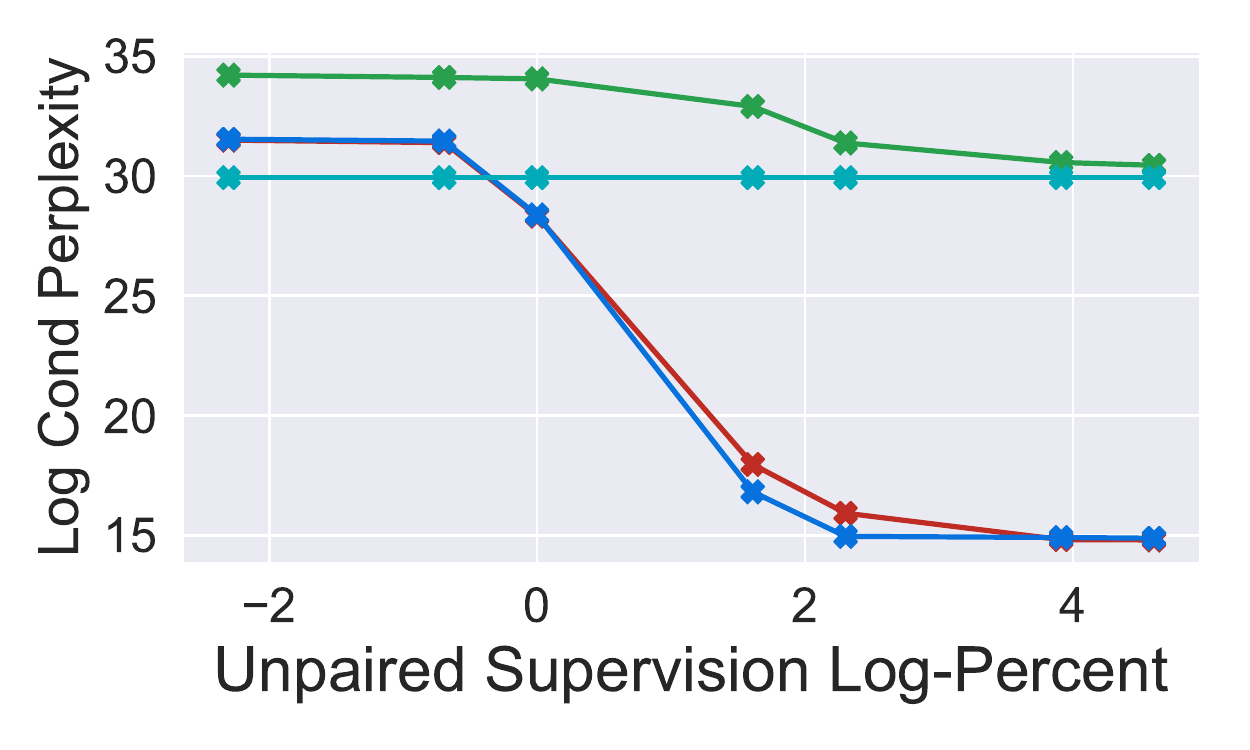}
      \caption{COCO}
  \end{subfigure}
  \begin{subfigure}[h]{0.24\textwidth}
      \centering
      \includegraphics[width=\linewidth]{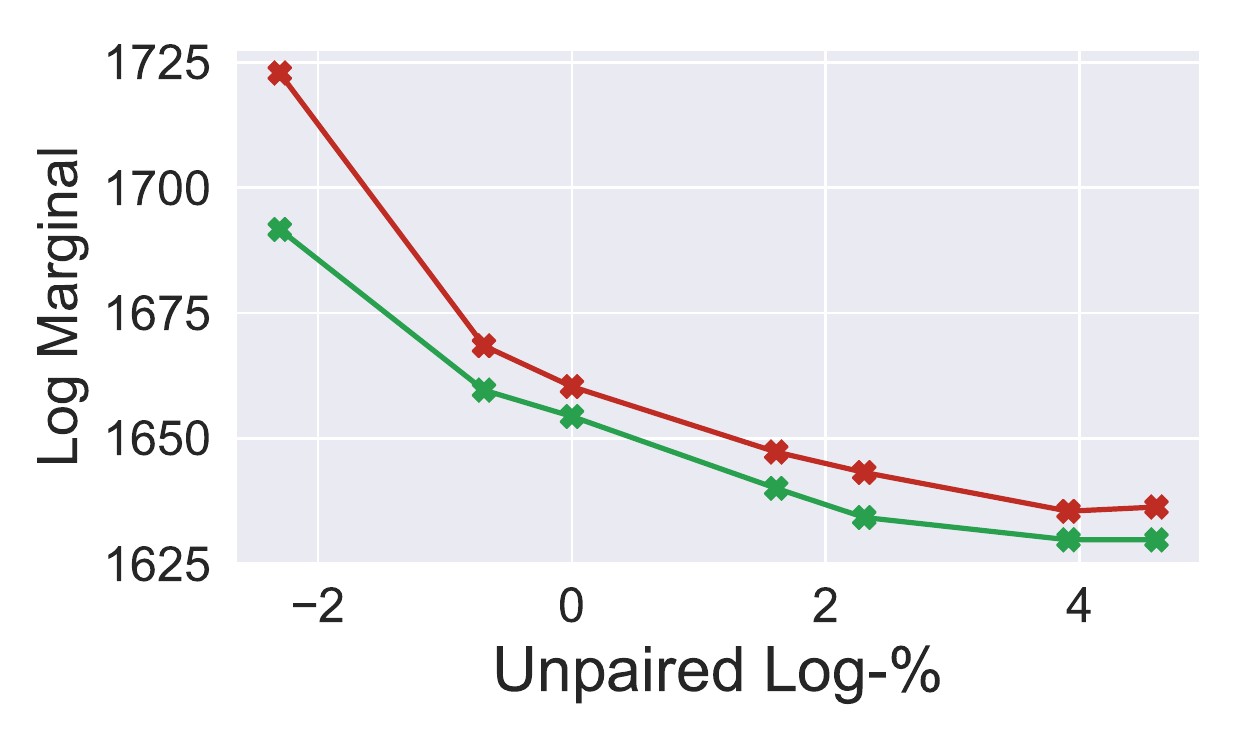}
      \caption{CelebA}
  \end{subfigure}
  \begin{subfigure}[h]{0.24\textwidth}
      \centering
      \includegraphics[width=\linewidth]{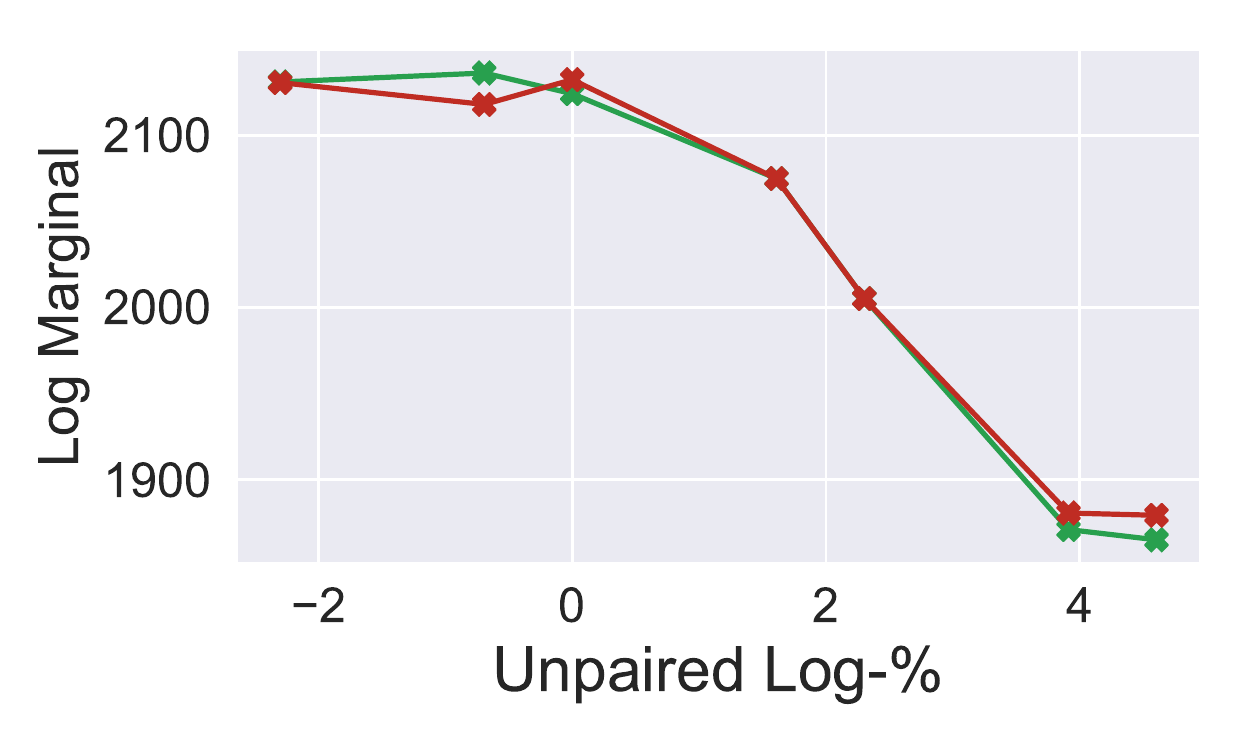}
      \caption{COCO}
  \end{subfigure}
  \caption{Effect of unpaired data on conditional prediction (a,b) and marginal (image) density estimation (c,d). Since supervised baseline is unable to use unpaired data, it is constant. The amount of paired supervision is at 1\% for all experiments here. }
  \label{fig:unsup1}
\end{figure*}

\subsection{Multimodal VAEGANs}
Recall in Sec.~\ref{sec:gan}, we showed that a KL divergence can be rewritten as a supremum over the difference of two expectation terms. As $\mathcal{L}_{\text{M}}$ is composed of four KL divergences, we can replace any of four terms with its Fenchel dual. For example, the divergence between two multimodal distributions can be lower bounded by
\begin{align}
  \KL(q_\phi(\vx,\vy,\vz)||r(\vx,\vy,\vz)) \geq \sup_{S \in \mathcal{S}}\{\mathbb{E}_{q_\phi(\vx,\vy,\vz)}[S(\vx,\vy,\vz)] - \log \mathbb{E}_{r(\vx,\vy,\vz)}[e^{S(\vx,\vy,\vz)-1}]\} \label{eq:vaegan:obj}
\end{align}
where $r(\vx,\vy,\vz)=q_\phi(\vx,\vz)p_\theta(\vy|\vx,\vz)$ and $\mathcal{S} = \{ S: \mathcal{X} \times \mathcal{Y} \times \mathcal{Z} \rightarrow \mathbb{R} \}$ specifies a family of multimodal functions analogous to the discriminator $\mathcal{T}$ from Eq.~\ref{eqn:gan}. We arrive at similar definitions for the other terms in Thm.~\ref{thm:1}. Drawing from \citep{dumoulin2016adversarially}, $q_\phi(\vz|\vx,\vy)$, $q_\phi(\vz|\vx)$, and $q_\phi(\vz|\vy)$ are adversarially learned inference (ALI) networks, commonly Gaussian posteriors. Sampling from the joint $q_\phi(\vx,\vy,\vz)$ amounts to sampling $\vz \sim q_\phi(\vz|\vx,\vy)$ with $\vx,\vy \sim p_d(\vx,\vy)$, the empirical distribution. Similarly, sampling from the product $q_\phi(\vx,\vz)p_\theta(\vy|\vx,\vz)$ amounts to sampling $\vz \sim q_\phi(\vz|\vx)$ with $\vx \sim p_d(\vx)$ and $\vy \sim p_\theta(\vy|\vz)$. As everything is reparameterizable, we can optimize with SGD. This allows us to define a multimodal objective by replacing any of terms in $\mathcal{L}_{\text{M}}$ with its dual. By doing so, we no longer explicitly model the likelihood term $p_\theta(\cdot|\vz)$, but expect improved sample quality.

\paragraph{Lipschitz Constraint}
GANs are known to struggle with mode collapse, where generated samples have low variety, memorizing individual images from the training dataset. In practice, we find this to be even more problematic in the multimodal domain. Wasserstein GANs (WGAN) \citep{arjovsky2017wasserstein,gulrajani2017improved} have been presented as a possible solution where instead of a variational divergence, the model minimizes the earth mover (EM) distance between the generative and empirical distributions. EM distance has its own dual representation that restricts the discriminator to be 1-Lipschitz, a constraint that prevents a discriminator from overpowering the generator, helping deter mode collapse \cite{arjovsky2017wasserstein}. Unfortunately, EM is not directly applicable to our family of objectives. However, using the VAEGAN as presented would indeed face mode collapse issues.

We propose a simple amendment inspired by WGAN: add a 1-Lipschitz constraint on the class of functions, $\mathcal{T}$ and $\mathcal{S}$ in Eq.~\ref{eq:vaegan:obj}. We note that this is still well-defined,
\begin{align}
  \KL(q_\phi(\vx,\vz)||p_\theta(\vx,\vz)) &\geq \sup_{T \in \mathcal{T}}\{\mathbb{E}_{q_\phi(\vx,\vz)}[T(\vx,\vz)] - \log \mathbb{E}_{p_\theta(\vx,\vz)}[e^{T(\vx,\vz)-1}]\}\label{eqn:lip:2}\\
  &\geq \sup_{T \in \mathcal{T}_{1}}\{\mathbb{E}_{q_\phi(\vx,\vz)}[T(\vx,\vz)] - \log \mathbb{E}_{p_\theta(\vx,\vz)}[e^{T(\vx,\vz)-1}]\}
  \label{eqn:lip:3}
\end{align}
where $\mathcal{T}_1$ represents a class of 1-Lipschitz functions. Notably, $\mathcal{T}_1 \subseteq \mathcal{T}$, meaning that the Eq.~\ref{eqn:lip:2} is lower bounded by Eq.~\ref{eqn:lip:3}. Thus, Eq.~\ref{eqn:lip:3} is still a lower bound on the original divergence. We can similarly define $\mathcal{S}_1 \subseteq \mathcal{S}$ for the multimodal discriminator. Fig.~\ref{fig:lipschitz} shows the effect of the Lipschitz constraint on sample quality for VAEGANs. Surprisingly, we find the constraint to reduce mode collapse and improve alignment between modalities.

\begin{figure*}[t]
\centering
  \begin{subfigure}[h]{0.24\textwidth}
      \centering
      \includegraphics[width=\linewidth]{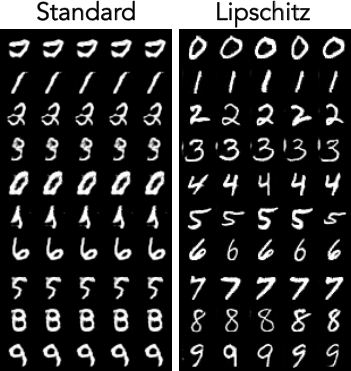}
      \caption{MNIST}
  \end{subfigure}
  \begin{subfigure}[h]{0.24\textwidth}
      \centering
      \includegraphics[width=\linewidth]{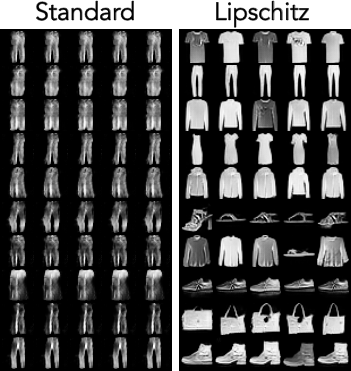}
      \caption{FashionMNIST}
  \end{subfigure}
  \begin{subfigure}[h]{0.24\textwidth}
      \centering
      \includegraphics[width=\linewidth]{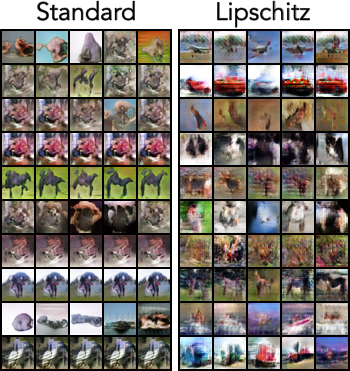}
      \caption{CIFAR10}
  \end{subfigure}
  \begin{subfigure}[h]{0.24\textwidth}
      \centering
      \includegraphics[width=\linewidth]{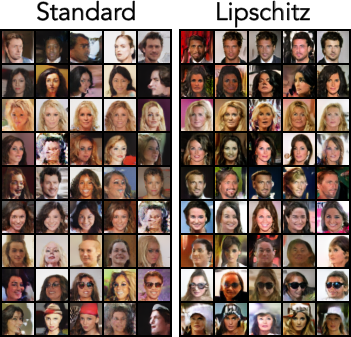}
      \caption{CelebA}
  \end{subfigure}
  \caption{Conditional samples from a VAEGAN for each label with and without the Lipschitz constraint. Without the constraint, we see evidence of mode collapse (repetitive images) and incorrect samples for certain classes (e.g. digits 4 and 5 in MNIST or all of FashionMNIST). }
  \label{fig:lipschitz}
\end{figure*}

\subsection{Multimodal VAE-Flows}
As flow models do exact inference, there is no need to use variational methods. But we can still decompose a multimodal flow model into a marginal and conditional component, namely
\begin{align}
\log p_\theta(\vx,\vy) &= \log p_\theta(\vx) + \log p_\theta(\vy|\vx)\\
&= \log p_\theta(\vy) + \log p_\theta(\vx|\vy)
\end{align}
Sec.~\ref{sec:flow} showed how to parameterize the marginal terms by defining a deterministic transformation $f_\theta: \mathcal{X} \rightarrow \mathcal{Z}$. We can do something similar for the conditional term:
\begin{align}
  \log p_\theta(\vx|\vy) &= \log p_\theta(\vz|\vy) + \log | \det \frac{\partial \vz}{\partial \vx^T} | \\
              &\approx \log q_\phi(\vz|\vy) + \log | \det \frac{\partial \vz}{\partial \vx^T} |\label{eqn:flow:2}
\end{align}
where $\vz = g_\theta(\vx)$ for a new ``flow" function $g:\mathcal{X} \rightarrow \mathcal{Z}$ parameterized by a neural network. Note that $g_\theta(\cdot)$ and $f_\theta(\cdot)$ are two distinct models and do not share parameters. In Eq.~\ref{eqn:flow:2}, we replaced the true (intractable) posterior with a variational approximation. Similarly, we can learn transformations for $\log p_\theta(\vx|\vy)$. In practice, we use real-valued non-volume preserving transforms, or real NVP \citep{dinh2014nice,dinh2016density}. These models design their flow functions using affine coupling layers, which enforces the Jacobian matrix to be triangular, thereby making the determinant tractable. We will denote these multimodal models as VAENVP. Future work could explore using recently proposed flows like Glow \citep{kingma2018glow} or FFJORD \citep{grathwohl2018ffjord}.



\section{VAEGAN and VAENVP Experiments}
\label{sec:hybrid_experiments}
We conduct several experiments using VAEs, GANs, or NVPs to capture image modalities whereas label and text, being discrete, requires VAEs.
With these new hybrid models, we focus on improved sample generation, especially for naturalistic images.

\subsection{Evaluation}
There is a lack of consistent high-quality metrics for sample generation, so we settle for a few heuristics. As in previous work, we compute the Frenet Inception distance (FID) \citep{salimans2016improved} between generated image samples and observations from the training datset. FID is defined as
\begin{equation}
\text{FID} = ||\mu_r - \mu_g||^2 + \text{Tr}(\Sigma_r + \Sigma_g - 2(\Sigma_r\Sigma_g)^{\frac{1}{2}})
\end{equation}
where $\mu_r, \Sigma_r$ are the mean and covariance of activations from an intermediate layer of a deep neural network (we call this an ``oracle" network) given inputs $\vx \sim p_d(\vx)$. The variables $\mu_g, \Sigma_g$ are similarly computed using inputs $\vx \sim p_\theta(\vx|\vz)$. Previous work use Inception \citep{szegedy2015going} trained on ImageNet to compute FID. In practice, we do not find it appropriate to use the same network on all datasets (e.g. Inception on MNIST), especially ones very different in distribution and content than ImageNet. Instead, we train our own supervised network specific to each dataset to compute FID. This implies that while our FID can be compared across models for a single dataset, it cannot be compared across datasets. Specifically, we use a dual path network (DPN) \citep{chen2017dual} as the oracle for image and label datasets. Since nothing restricts FID from being used for image and text datasets, we pull statistics from an intermediate layer from a trained captioner with visual attention \citep{xu2015show}. In all settings, we compute FID using 50k samples.

Given an oracle network, we can separately use it classify or caption \textit{conditionally} sampled images from a multimodal generative model. For instance, given a generated image of a MNIST digit with label 9, we can use the oracle to classify the image. Assuming the oracle is a ``perfect" classifier, a better generative model would produce samples that can be classified correctly. A perfect generator would achieve a score of 1. Likewise, for image and text datasets, a perfect generator should achieve very low perplexity as scored by the oracle. We call this metric the \textit{oracle sampling error} or OSE. In short, FID is a measure of marginal sample quality whereas OSE is a measure of conditional sample quality.

\begin{table*}[t!]
\centering
\footnotesize
\begin{subtable}[t]{\textwidth}
\centering
\begin{tabular}{l|cc|cc|cc|cc}
    \toprule
    & \multicolumn{2}{c}{MNIST} & \multicolumn{2}{c}{FashionMNIST} & \multicolumn{2}{c}{CIFAR10} & \multicolumn{2}{c}{CelebA} \\
    Model & FID & OSE & FID & OSE & FID & OSE & FID & OSE \\
    \midrule
    VAE & $1.31$ & - & $1.35$ & - & $8.48$ & - & $8.86$ & - \\
    GAN & $1.77$ & - & $\mathbf{1.32}$ & - & $\mathbf{2.35}$ & - & $\mathbf{3.39}$ & - \\
    GAN$^\S$ & $1.87$ & - & $2.24$ & - & $3.03$ & - & $4.19$ & - \\
    NVP & $2.17$ & - & $5.01$ & - & $3.86$ & - & $5.05$ & - \\
    CVAE & - & $0.91$ & - & $0.82$ & - & $0.18$ & - & $0.89$ \\
    CGAN & - & $0.95$ & - & $0.71$ & - & $0.39$ & - & $0.89$ \\
    BiVCCA & $4.41$ & $0.90$ & $5.51$ & $0.75$ & $8.88$ & $0.09$ & $18.2$ & $0.90$ \\
    MVAE & $1.60$ & $0.64$ & $1.43$ & $0.53$ & $8.42$ & $0.11$ & $9.09$ & $0.83$ \\
    JMVAE & $2.60$ & $0.80$ & $1.79$ & $0.74$ & $8.61$ & $0.18$ & $8.61$ & $0.87$ \\
    VAEVAE & $1.88$ & $0.89$ & $1.68$ & $0.77$ & $8.48$ & $0.15$ & $8.53$ & $0.89$ \\
    VAEVAE$^\ddagger$ & $1.76$ & $0.89$ & $1.67$ & $0.76$ & $8.51$ & $0.16$ & $8.89$ & $0.90$ \\
    VAEGAN & $12.8$ & $0.50$ & $35.7$ & $0.10$ & $5.92$ & $0.34$ & $18.7$ & $0.89$ \\
    VAEGAN$^\ddagger$ & $11.4$ & $0.59$ & $19.2$ & $0.37$ & $3.13$ & $0.37$ & $13.7$ & $0.88$ \\
    VAEGAN$^\S$ & $\mathbf{0.40}$ & $\mathbf{0.99}$ & $1.44$ & $\mathbf{0.93}$ & $2.80$ & $\mathbf{0.44}$ & $4.07$ & $\mathbf{0.93}$ \\
    VAENVP & $3.14$ & $0.94$ & $4.59$ & $0.76$ & $4.13$ & $0.32$ & $5.71$ & $\mathbf{0.93}$ \\
    \bottomrule
\end{tabular}
\end{subtable}
\newline\vspace{2em}\newline
\begin{subtable}[t]{\textwidth}
\centering
\begin{tabular}{l|cc|cc|cc|cc}
    \toprule
    & \multicolumn{2}{c}{MNIST Math} & \multicolumn{2}{c}{CIC} & \multicolumn{2}{c}{Flickr30k} & \multicolumn{2}{c}{COCO}\\
    Model & FID & OSE & FID & OSE & FID & OSE & FID & OSE \\
    \midrule
    VAE & $4.68$ & - & $20.3$ & - & $2.7e2$ & - & $3.3e2$ & - \\
    GAN & $3.60$ & - & $6.88$ & - & $88.8$ & - & $53.2$ & - \\
    GAN$^\S$ & $3.75$ & - & $7.21$ & - & $1.6e2$ & - & $1.3e2$ & - \\
    NVP & $6.44$ & - & $13.0$ & - & $88.9$ & - & $80.9$ & - \\
    CVAE & - & $4.1e4$ & - & $5e13$ & - & $2.6e7$ & - & $1.8e10$ \\
    CGAN & - & $3.5e4$ & - & $3e13$ & - & $1.2e7$ & - & $4.9e9$ \\
    BiVCCA & $4.62$ & $3.7e4$ & $38.2$ & $3e11$ & $3.0e2$ & $5.8e6$ & $3.9e2$ & $1.5e10$ \\
    MVAE & $4.18$ & $4.0e4$ & $25.0$ & $2e13$ & $3.5e2$ & $4.3e7$ & $4.6e2$ & $1.3e10$ \\
    JMVAE & $4.39$ & $3.9e4$ & $45.1$ & $3e13$ & $3.1e2$ & $5.2e7$ & $3.5e2$ & $1.9e10$ \\
    VAEVAE & $3.79$ & $3.8e4$ & $21.3$ & $3e11$ & $2.6e2$ & $1.2e7$ & $2.9e2$ & $1.2e10$ \\
    VAEVAE$^\ddagger$ & $3.88$ & $3.9e4$ & $15.3$ & $2e11$ & $2.6e2$ & $1.8e7$ & $3.0e2$ & $1.3e10$ \\
    VAEGAN & $\mathbf{1.12}$ & $4.5e4$ & $31.0$ & $4e13$ & $17.2$ & $4.7e6$ & $\mathbf{7.41}$ & $7.7e8$ \\
    VAEGAN$^\ddagger$ & $2.00$ & $4.7e4$ & $34.4$ & $3e13$ & $\mathbf{9.17}$ & $\mathbf{1.1e6}$ & $10.6$ & $\mathbf{5.6e8}$ \\
    VAEGAN$^\S$ & $1.74$ & $\mathbf{3.4e4}$ & $\mathbf{3.01}$ & $\mathbf{1e10}$ & $15.7$ & $4.1e6$ & $10.9$ & $6.2e8$ \\
    VAENVP & $4.98$ & $3.8e4$ & $9.46$ & $2e10$ & $76.9$ & $9.8e6$ & $57.8$ & $4.4e9$ \\
    \bottomrule
\end{tabular}
\end{subtable}
\caption{Sample Evaluation of a suite of Unimodal and Multimodal Models. The superscript $\ddagger$ indicates using  Product-of-Experts in the variational posterior. The superscript $\S$ indicates using the Lipschitz constraint in training. We use the Frenet Inception distance (FID) to measure the quality of samples from the marginal distribution. We define a new evaluation metric : oracle sampling error (OSE) to measure the quality of samples from a conditional distribution. OSE measures the classification error or conditional perplexity of generated images conditioned on labels or text, respectively. }
\label{table:image_label_text:hybrid}
\end{table*}

\subsection{Results}

Table~\ref{table:image_label_text:hybrid} show FID and OSE for a suite of popular unimodal and multimodal generative models. We include VAE, GAN, and NVP as lower bounds for FID and include CVAE and CGAN as upper bounds for OSE. In the table, the superscript $\S$ means the model is trained with a Lipschitz constraint; the superscript $\ddagger$ means the model is trained with PoE. In both tables, VAEGAN achieves one of (if not) the best (lowest) FID and best (highest) OSE. For simpler image domains, the sample quality of VAE, NVP, and GAN models are roughly equal (e.g. MNIST), as expected. For more complex image domains, we can see the blurriness of VAEs represented as high FIDs and low OSEs. Generally, flow models (NVP and VAENVP) are better than VAE but worse than GANs. The experiments also highlight that the Lipschitz constraint is extremely important for VAEGAN models. Without it, VAEGAN suffers greatly from mode collapse, as evidenced by very high FIDs and low OSEs (see MNIST, CIFAR, CelebA, CIC). With the Lipschitz constraint, VAEGAN$^{\S}$ achieves lower FID than the unimodal baselines and higher OSE than CVAE and CGAN. The impact of the Lipschitz constraint is much less noticeable with the unimodal GAN, sometimes even lowering performance.

In the supplement, we include random samples from our generative models (see Fig.~\ref{fig:imagelabel_samples} and \ref{fig:imagetext_samples}).
For CelebA, we start to see the artifacts of each model: VAEVAE samples have blurry background, VAEGAN samples look similar to each other, and VAENVP samples appear distorted by wave-like noise. For naturalistic images like those in CIFAR10, Flickr30k, and COCO, we see a stark difference: VAEGAN samples look sharp whereas VAEVAE and VAENVP samples are dominated by artifacts. Finally, we also include randomly-chosen conditional captions. While not perfect, they do often faithfully describe the scene.

\section{Abstraction of Visual Features}

A key motivation for multimodal generative models is the hypothesis that sharing statistics across diverse modalities will encourage more abstract representations.
Having now explored many examples of learning with images and text, we wanted to investigate the effect of language on the learned visual representations. Natural language is intrinsically compositional, emphasizing the relationships between objects and attributes in a scene. This suggests that visual representations influenced by language should also be more abstract, focusing on objects, their properties, and relations.

In order to coherently ask whether visual representations influenced by language are more abstract than those without language, we first need a way to learn representations comparably from images alone or images and language together. Our family of multimodal models provide this setting.
We evaluate visual representations which are the mean of the posterior specified by the image inference network $q_\phi(\vz|\vx)$, after unimodal or multimodal training.
Concretely, we compare VAEVAE and VAEGAN features against unimodal image representations from VAEs and GANs (trained without language). We do \textit{not} use the joint representation $q_\phi(\vz|\vx,\vy)$ that explicitly embeds language.

Next, we need a way to evaluate relevant abstraction in the learned representations. Measuring abstraction itself is not easy, so we use three  proxies. First, we directly approximate the compositionality of continuous visual representations, using a recently proposed measure.
Second, we measure improvements in bounding box prediction with our ``multimodal" image features, as a way of testing when learned representations are more ``object centric''.
Finally, we do the same for visual relation prediction, as a way of testing when learned representations are more ``relational''.

\subsection{Compositionality Estimation}

Tree Reconstruction Error (TRE) \citep{andreas2019measuring} is a recently proposed metric for compositionality on vector representations, when the underlying compositional structure of the domain is known.
We imagine having an ``oracle" that maps each input image to a composition of primitive objects and relations (a decorated tree).
We then learn vector embeddings for the primitives, and potentially the composition operator, that when composed along the tree best approximate the image representation we are evaluating.
A perfect approximation would indicate that the target representations exactly reflect the structure produced by the oracle. Hence we use the error between the composed representation and the target representation as a measure of (or lack of) compositionality: A smaller distance suggests more compositional structure.

We compute TRE on the ShapeWorld dataset\footnote{We use a simplified Python implementation of ShapeWorld (\url{https://github.com/jayelm/minishapeworld}) that allows us to store colors, positions, rotations, and sizes of rendered shapes.} \citep{kuhnle2017shapeworld} which contains images of geometric shapes with different colors, positions, rotations, and sizes along with captions describing the scene (e.g.~``A red square above a blue circle'').
\begin{figure}[h!]
  \centering
  \includegraphics[width=0.5\linewidth]{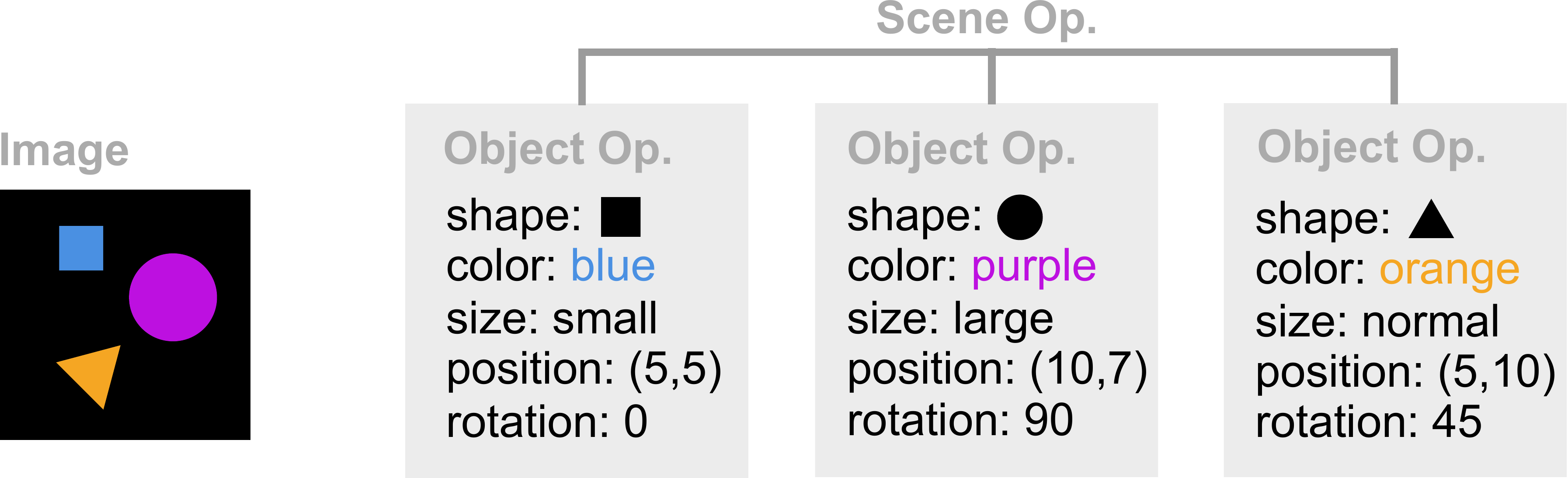}
  \caption{Setup for the oracle representation used to compute TRE. The ``Object" operator is a small neural network over the features of the object. The ``Scene" operator is a deterministic average.}
  \label{fig:trecomp:fig}
\end{figure}

We initialize a vector embedding for each possible color, shape, size, x coordinate, and y coordinate. Then, we define two compositional operators: first, the ``Object" operator ingests the vectorized color, shape, size, and position and returns a single ``Object" vector. We parameterize the ``Object" operator as a MLP of varying depth.
(\citep{andreas2019measuring} assumed fixed composition by averaging, which corresponds to the special case where the ``Object'' MLP is linear.)
Second, the ``Scene" operator averages a set of ``Object" vectors into a ``Scene" vector. See Fig.~\ref{fig:trecomp:fig} for an example. The primitive embeddings and the parameters of the ``Object" operator are optimized with SGD using the distance between our image representation and the corresponding ``Scene" vector as an objective. We experiment with L$_1$ and L$_2$ distance. Constructing the oracle representation in this manner ensures an object-oriented representation. Thus, we can treat the TRE score with respect to the variational image representation as a measure of its abstraction.

Table~\ref{table:tre} measures how well VAEVAE and VAEGAN image representations can be reconstructed from the oracle representation. Compared to unimodal baselines, our model family has consistently lower TRE scores. This indicates that the representations learned mutimodally better reflect the underlying object features of this domain, abstracting away from the lower level visual information. Also, VAEGAN consistently outperforms VAEVAE, which suggests that a stronger visual model is important for capturing compositionality. As we vary the expressivity of the ``Object" MLP from purely linear to three layers with ReLU, the patterns remain consistent. In the supplement Sec.~\ref{sec:sup:compexp}, we explore a weaker version of TRE where we use a bag-of-words feature set to represent language (e.g. color of triangles, number of triangles, color of squares, etc) with summation as the composition operator, as done in \citep{andreas2019measuring}. We find similar patterns to Table~\ref{table:tre}.

\begin{table}[h!]
\centering
\footnotesize
\begin{tabular}{l|l|c|c}
    \toprule
    Model & Oracle & TRE (L$_2$) & TRE (L$_1$) \\
    \midrule
    VAE & Linear & $4.468 \pm 0.011$ & $21.489 \pm 0.058$ \\
    GAN+ALI & Linear & $4.606 \pm 0.011$ & $19.507 \pm 0.051$ \\
    VAEVAE & Linear & $3.686 \pm 0.009$ & $16.188 \pm 0.046$ \\
    VAEGAN & Linear & $\mathbf{2.049 \pm 0.008}$ & $\mathbf{10.387 \pm 0.043}$ \\
    \midrule
    VAE & 2-MLP & $2.806 \pm 0.008$ & $13.776 \pm 0.042$ \\
    GAN+ALI & 2-MLP & $2.397 \pm 0.008$ & $11.099 \pm 0.033$ \\
    VAEVAE & 2-MLP & $1.955 \pm 0.008$ & $10.355 \pm 0.035$ \\
    VAEGAN & 2-MLP & $\mathbf{1.874 \pm 0.007}$ & $\mathbf{9.595 \pm 0.037}$ \\
    \midrule
    VAE & 3-MLP & $2.605 \pm 0.008$ & $12.688 \pm 0.039$ \\
    GAN+ALI & 3-MLP & $2.239 \pm 0.007$ & $10.535 \pm 0.033$ \\
    VAEVAE & 3-MLP & $1.835 \pm 0.008$ & $8.743 \pm 0.031$ \\
    VAEGAN & 3-MLP & $\mathbf{1.617 \pm 0.006}$ & $\mathbf{8.212 \pm 0.029}$ \\
    \bottomrule
\end{tabular}
\caption{Measuring compositionality on image representations $q_\phi(\vz|\vx)$ using the TRE metric and the MiniShapeWorld dataset: a lower number is better. Image features learned using VAEVAE and VAEGAN models report better TRE scores.}
\label{table:tre}
\end{table}

\subsection{Object Detection}

In addition to measuring compositionality directly, we can infer the impact of language on visual representations through transfer tasks that require knowledge of objects, relations, and attributes. Here, we consider bounding box prediction. The hypothesis is that image representations learned in a multimodal setting should contain more information about objects, since language aptly describes the objects in a scene. This should make it easier to identify the locations of objects in an image.

To test this hypothesis, we first take the COCO 2017 \citep{chen2015microsoft} dataset and pretrain a multimodal generative model e.g. VAEVAE or VAEGAN. Next, we train a well-known single-shot detection model, YOLOv3\footnote{We use a public PyTorch implementation of YOLOv3 with a pretrained DarkNet53 backbone: \url{https://github.com/DeNA/PyTorch_YOLOv3}. All hyperparameters are kept as is.} \citep{redmon2018yolov3,redmon2016you} with the following edit: using a linear layer we expand the multimodal embedding to be $(1, S, S)$ where $S$ is the maximum input size at the current YOLO layer. Since it is common to train YOLOv3 with multiscale inputs, we use an adaptive pooling layer to reshape the embedding to be $(1, S_{\text{cur}}, S_{\text{cur}})$ where $S_{\text{cur}}$ is the size of the inputs in the current minibatch and layer. At each residual block with input of shape $(F, S_{\text{cur}}, S_{\text{cur}})$, we append the multimodal filter followed by a 1x1 convolution to project the input back to $F$ filters. Since there are several residual blocks in the YOLOv3 model, the intuition of the design to repeatedly offer information from the multimodal embedding to the detection model. The parameters of the generative model are frozen so training YOLO does not change the multimodal representation. We train for 500k iterations and measure AP[IoU=0.5] and AP[IoU=0.5:0.95] every 5000 iterations on the test set.
\begin{figure}[h]
  \centering
  \begin{subfigure}[h]{0.40\linewidth}
      \centering
      \includegraphics[width=\linewidth]{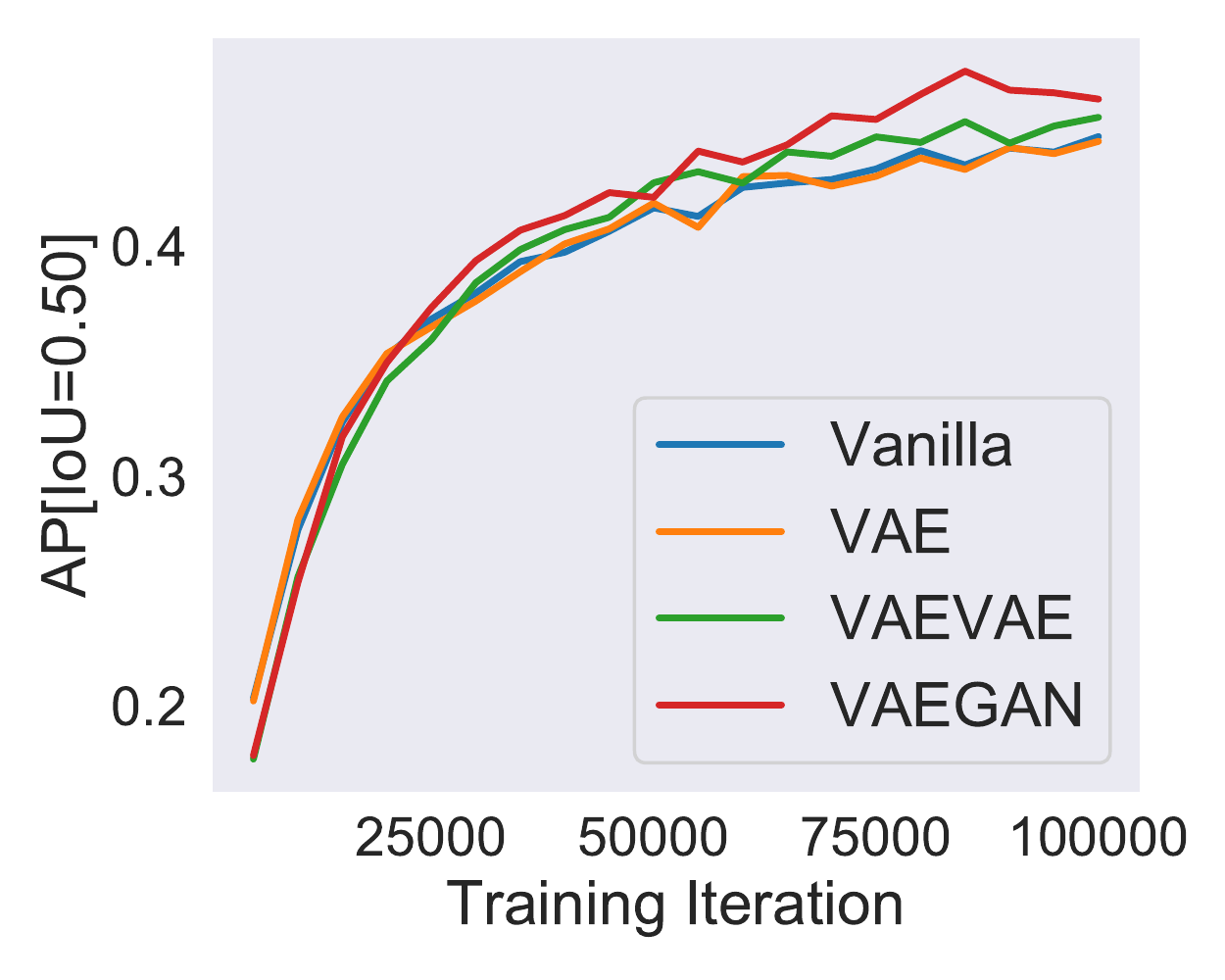}
      \caption{}
  \end{subfigure}
  \begin{subfigure}[h]{0.40\linewidth}
      \centering
      \includegraphics[width=\linewidth]{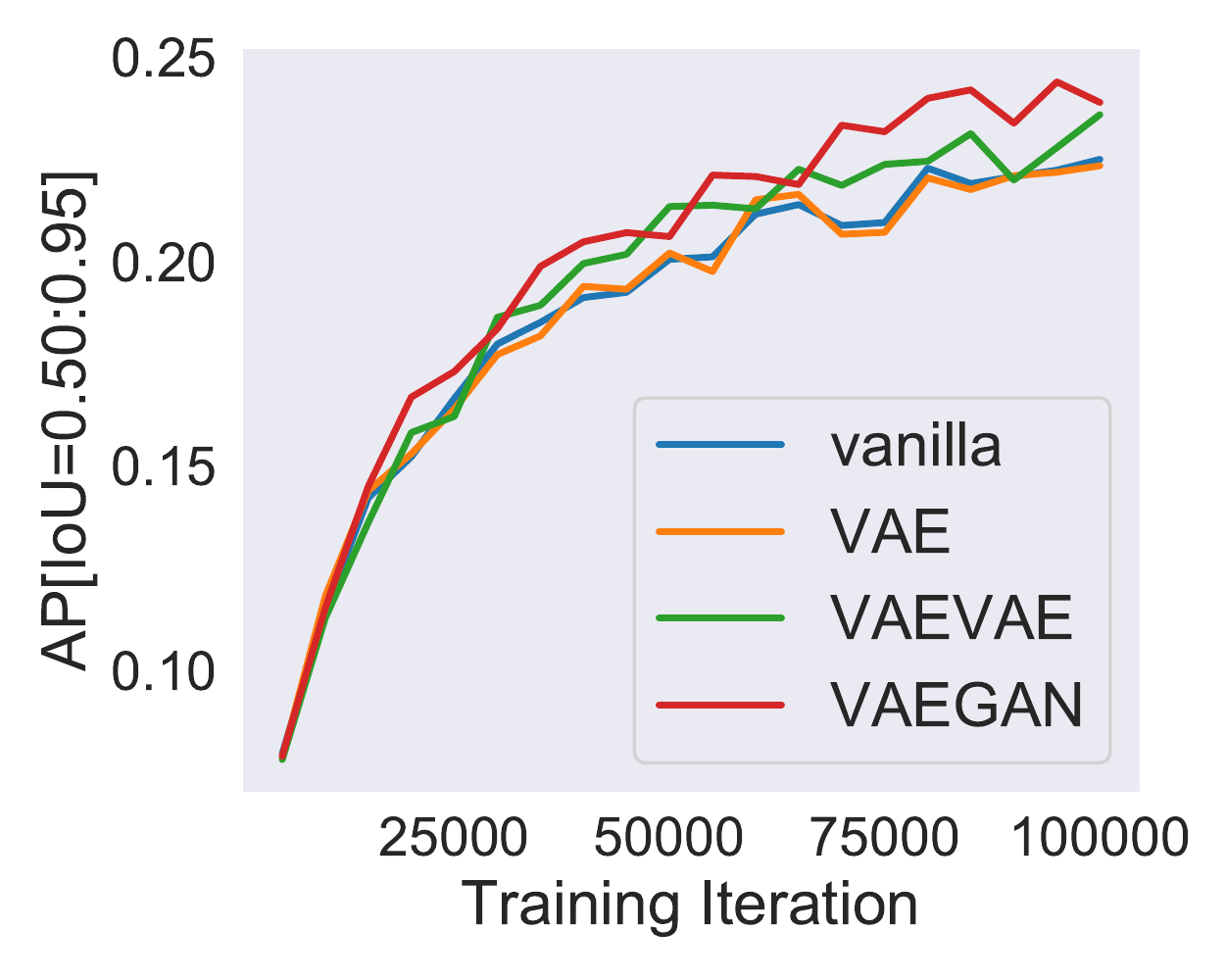}
      \caption{}
  \end{subfigure}
  \caption{Comparison of trajectories showing average precisions measured every 5000 training iterations. Adding VAEVAE and VAEGAN image embeddings from $q(\vz|\vx)$ (note that no language is explicitly given) to the YOLOv3 model improves learning speed, reaching higher evaluation metrics earlier in training.}
  \label{fig:downstream}
\end{figure}

In Fig.~\ref{fig:downstream}, we find YOLOv3 with ``multimodal" image features learns faster, evidenced by the red (VAEGAN) and green (VAEVAE) curves bounding the blue (standard) one from above. Besides the vanilla YOLOv3, we also include a baseline using image representations learned without language, from a VAE. Notably, the plots show that this baseline was almost identical in training trajectory to the vanilla YOLOv3 model, which is sensible since YOLOv3 already includes pretrained Darknet image features. Since DarkNet53 is a much larger model than the VAE, we do not expect the VAE image features do not offer any additional information. These results are encouraging, as they suggest that multimodal features can help us learn to perform complex visual tasks more efficiently. In many ways, language describing an image acts as strong supervision signal for any task involving the objects and relations in the image. Rather than initializing our bounding box detection models from image features alone (which may be less compositional), we  find that we need fewer epochs or fewer parameters by using a multimodal initialization. For large models like YOLOv3 which take significant computation, faster learning can make a large impact.
\begin{table*}[t!]
\centering
\scriptsize
\begin{tabular}{lccccccccc}
    \toprule
    & \multicolumn{3}{c}{Scene Graph Det.} & \multicolumn{3}{c}{Scene Graph Cls.} & \multicolumn{3}{c}{Predicate Cls.}\\
    Recall at & 20 & 50 & 100 & 20 & 50 & 100 & 20 & 50 & 100 \\
    \midrule
    LSVRD \citep{zhang2019large} & 12.5 & 18.5 & 23.3 & 22.8 & 23.2 & 23.2 & 45.7 & 46.5 & 46.5 \\
    LSVRD + VAE & 12.4 & 18.3 & 23.0 & 22.9 & 23.3 & 23.3 & 45.7 & 46.6 & 46.6 \\
    LSVRD + VAEGAN & 12.5 & 18.5 & 23.4 & 33.1 & 33.9 & 33.9 & 55.0 & 56.1 & 56.2 \\
    LSVRD + Word2Vec \citep{zhang2019large} & 20.2 & \textbf{27.4} & 32.1 & 36.2 & 36.9 & 36.9 & 67.4 & 69.0 & 69.0 \\
    LSVRD + VAE + Word2Vec & 20.2 & 27.1 & 32.0 & 36.0 & 36.7 & 36.7 & 67.5 & 68.8 & 68.8 \\
    LSVRD + VAEGAN + Word2Vec & 20.2 & 27.3 & \textbf{32.2} & \textbf{36.7} & \textbf{37.2} & \textbf{37.2} & \textbf{67.9} & \textbf{69.3} & \textbf{69.3} \\
    \bottomrule
\end{tabular}
\caption{Comparison with one of the state-of-the-art visual relation models on the VG200 dataset. In particular, we measure the effect of VAEGAN image representations (trained in a multimodal setting) on the ``Large Scale VRD" architecture.}
\label{table:vrd_results}
\end{table*}

\subsection{Visual Relation Prediction}
The second transfer task we consider is relation prediction (e.g. is the person driving the car?), of which a sub-task is object detection. Since we have shown promising results with predicting bounding boxes, we might hope that language describing objects and relations in an image can have a similar effect here.

We use the Visual Genome dataset \citep{krishna2017visual}, which contains images with annotated object labels, bounding boxes, and captions describing the objects in the scene. We first train the VAEVAE or VAEGAN on the image and captions alone, after which we only keep the image encoder, $q_\phi(\vz|\vx)$ to extract representations. For relation prediction, we use a popular subset of Visual Genome called VG200 \cite{xu2017scene,newell2017pixels,zellers2018neural,yang2018graph,zhang2019large}. VG200 contains 150 object categories and 50 predicate categories.

We pick a state-of-the-art relation predictor built on the MaskRCNN architecture \cite{he2017mask}, denoted LSVRD\footnote{We use the authors' PyTorch implementation of LSVRD: \url{https://github.com/facebookresearch/Large-Scale-VRD}. All hyperparameters are kept as is.} \citep{zhang2018large} and edit the architecture to include the multimodal embedding. The LSVRD architecture contains two RCNN components, whose outputs are fed into the different heads for classification, detection, and relation prediction. For each RCNN, we combine the output of its base convolutional body (call this $a$) with the multimodal embedding in a similar fashion as above: assuming $a$ has $F$ filters, the embedding is resized to a single filter of matching width and height to $a$ (this is done via adaptive pooling to handle multiscale inputs). Then, concatenating the multimodal filter with $a$ for a total of $F+1$ filters, we use a 1x1 convolutional layer to output $F$ filters. The rest of the LSVRD network is untouched. As in \citep{zhang2018large}, during training, the true bounding box and object classes are given to the model. For evaluation, \citeauthor{zhang2018large} compute three metrics: scene graph detection (SGDET) where neither the bounding box nor the object class is given, scene graph classifcation (SGCLS) where only the bounding box is given, and predicate classification (PRDCLS) where both are given.

\citeauthor{zhang2018large} already explores directly adding  language features by embedding the object and predicate names with Word2Vec \citep{mikolov2013efficient}, which we find in Table~\ref{table:vrd_results} to significantly improve performance. Word2vec provides a good prior over which relations are semantically plausible, which is especially important given the sparse labels in VG200. We compare LSVRD with and without Word2vec to LSVRD with multimodal embeddings, the latter of which we expect to provide compositional image features that may positively impact learning.

Table~\ref{table:vrd_results} show results for six variations of LSVRD. First, we were able to reproduce the improvements in performance from using Word2Vec embeddings. Second, we point out that adding VAE embeddings to LSVRD did not impact performance -- since LSVRD uses a pretrained object detection model that captures most of the salient visual features already. However, adding VAEGAN image features trained with the presence of natural language leads to significant improvements in  SGCLS and PRDCLS (around 10\% without Word2Vec). Critically, no explicit language features are given, the VAEGAN features are from the image encoder alone. Despite that, the performance of LSVRD + VAEGAN close the gap to using pre-trained Word2Vec embeddings. Given the previous evidence of compositionality in our multimodal representations, we suspect the VAEGAN image fatures to contain information describing objects and relations in an image imbued by language during pre-training. Finally, we separately train a version of LSVRD with both Word2Vec and VAEGAN representations. Since VAEGAN representations were exposed to complete sentences rather than individual words, we might expect further improvements. The last row of Table~\ref{table:vrd_results} shows a modest but consistent 0.3 to 0.5 percent increase in SGCLS and PRDCLS over a SOTA visual relation model.


\section{Related Work}
\label{sec:related_work}
In recent years, there has been a series of multimodal deep generative models. We provide an overview of a few below.

\paragraph{Conditional Models}
Many of the earlier multimodal generative models focus on conditional distributions. For example, given two modalities $\vx$ and $\vy$, conditional VAEs (CVAE) \citep{sohn2015learning} lower bound $\log p_\theta(\vx|\vy)$ where $\vy$ is often labels \citep{radford2015unsupervised}, attributes \citep{yan2016attribute2image}, text \citep{reed2016generative}, or another image \citep{isola2017image}. Similarly, conditional GANs (CGAN) \citep{mirza2014conditional} facilitate sample generation conditioned on another modality. This family of models are not truly \textit{multimodal} as they are not bi-directional and avoid capturing marginal and joint distributions.

\paragraph{Joint Models} Many of the generative models over two modalities share the same graphical model, yet differ in the objective. First, the bidirectional variational canonical correlation analysis (BiVCCA) \citep{wang2016deep} optimizes a linear combination of ELBOs using the two unimodal variational posteriors:
\begin{align*}
  \mathcal{L}_{\text{BiVCCA}} &= \mathbb{E}_{p_d(\vx,\vy)}\left(\omega \mathbb{E}_{q_\phi(\vz|\vx)}\left[\log \frac{p_\theta(\vx,\vy,\vz)}{q_\phi(\vz|\vx)}\right] + (1-\omega)\mathbb{E}_{q_\phi(\vz|\vy)}\left[\log \frac{p_\theta(\vx,\vy,\vz)}{q_\phi(\vz|\vy)}\right]\right)
\end{align*}
where $\omega \in [0, 1]$ is a weight. In particular, this model allows for bi-directional translation between modalities. However, it avoids specifying a joint variational posterior, which places a lot of the burden to do good inference on unimodal posteriors, which is difficult for very distinct modalities.

Variations on multimodal VAEs (MVAE) \citep{vedantam2017generative,wu2018multimodal} do include a multimodal variational posterior, $q_\phi(\vz|\vx,\vy)$. For example, \citep{wu2018multimodal} optimize a sum of three lower bounds:
\begin{align*}
\mathcal{L}_{\text{MVAE}} &= \text{ELBO}_{\phi,\theta}(\vx,\vy) + \text{ELBO}_{\phi,\theta}(\vx) + \text{ELBO}_{\phi,\theta}(\vy) \\
&\leq \log p_\theta(\vx,\vy) + \log p_\theta(\vx) + \log p_\theta(\vy)
\label{eqn:mvae}
\end{align*}
where the unimodal ELBOs require unimodal  posteriors: $q_\phi(\vz|\vx)$ and $q_\phi(\vz|\vy)$. Specifically, \citep{wu2018multimodal} define the $q_\phi(\vz|\vx,\vy)$ as the product of the unimodal posteriors.
  \begin{align*}
    q_\phi(\vz|\vx,\vy) \propto q_\phi(\vz|\vx)q_\phi(\vz|\vy)p(\vz)
  \end{align*}
assuming $q_\phi(\vz|\cdot)$ specifies a Gaussian distribution and a product of Gaussians is itself Gaussian. While this is an elegant formulation, we note that $\mathcal{L}_{\text{MVAE}}$ double counts the marginal distributions.
\begin{align*}
  \mathcal{L}_{\text{MVAE}} &\leq 2\log p_\theta(\vx) + \log p_\theta(\vy) +\log p_\theta(\vy|\vx) \\
  &= \log p_\theta(\vx) + 2\log p_\theta(\vy) +\log p_\theta(\vx|\vy)
\end{align*}
This could have undesirable consequences in prioritizing the marginals over the conditionals, leading to poor cross-modality sampling. Finally, another similar model is the joint multimodal VAE (JMVAE) \citep{suzuki2016joint}, which optimizes:
\begin{align*}
  \mathcal{L}_{\text{JMVAE}} &= \mathbb{E}_{p_d(\vx,\vy)}\left[\text{ELBO}_{\phi,\theta}(\vx,\vy) - \omega\KL(q_\phi(\vz|\vx,\vy)||q_\phi(\vz|\vx)) - \omega\KL(q_\phi(\vz|\vx,\vy)||q_\phi(\vz|\vy))\right]
\end{align*}
where $\omega \in \mathbb{R}^+$ is a hyperparameter.
Critically, this is very similar to the two conditional divergence terms in our multimodal objective, $\mathcal{L}_{\text{M}}$. In \citep{suzuki2016joint}, the authors have a similar derivation to Lemma~\ref{lem:3} that shows that the JMVAE is bounded:
\begin{equation*}
   \mathcal{L}_{\text{JMVAE}} \leq \log p_\theta(\vx|\vy) +\log p_\theta(\vy|\vx)
\end{equation*}
Almost opposite to the MVAEs, the JMVAE prioritizes the conditionals over the marginals. Further, it does not faithfully lower bound the joint probability, $\log p_\theta(\vx,\vy)$.

Our class of generative models neither over-prioritizes the marginals or conditionals and forms a valid lower bound on the joint density, $\log p_\theta(\vx,\vy)$. Rather than an ad-hoc solution to missing modalities, we formally motivate the KL divergence between the joint and unimodal variational posteriors by decomposing the joint variational divergence into a marginal and conditional term. Lastly, our class of models is  able to extend to likelihood-free models, which proved useful in realistic settings.

\paragraph{Hybrid Models}
A number of models combine VAEs with GANs and Flows in the unimodal setting. \cite{larsen2015autoencoding} share parameters between the GAN generator and VAE decoder to learn more robust feature representations. \cite{edupuganti2019vae} use a similar model to generate crisp images of MRI images. \cite{wan2017crossing} train a VAE and GAN (with ALI) on a bimodal task of reconstructing hand poses and depth images with a learned mapping between the VAE and GAN latent spaces. On the other hand, VAEs and flows have been combined (in the unimodal setting) to transform the variational posterior to a more complex family of distributions. \cite{rezende2015variational,tomczak2016improving,kingma2016improved} apply a series of $k$ invertible transformations to $z \sim q_\phi(z|x)$ such that $f_{\theta}^{k}(z)$ is distributed to more expressive distribution, leading to better density estimation and sample quality. To the best of our knowledge, our class of models is the first to combine VAEs, GANs, and flow models in the multimodal setting. Further, applying such models to discrete and continuous modalities, especially images and text, is novel.

\section{Conclusion}
In this paper we introduced an objective for training a family of multimodal generative models. We performed a series of experiments measuring performance on marginal density estimation and translation between modalities. We found appealing results across a set of multimodal domains, including many datasets with images and text which has not been explored before in the generative setting. In particular, we found that combining different generative families like VAEs and GANs led to strong multimodal results. Finally, we conducted a thorough analysis of the effects of missing data and weak supervision on multimodal learning. We found the proposed VAEVAE and VAEGAN to outperform other generative models under high levels of missing data. Moreover, latent features learned from image distributions under the presence of language capture a notion of compositionality that led to better performance in downstream tasks. Future work could investigate applying these ideas to modelling video \citep{tan2019factorized}, where each frame contains visual, audio, and textual information.

\bibliography{multimodal}
\bibliographystyle{icml2018}

\newpage
\appendix
\onecolumn

\section{Datasets and Hyperparameters}

For all experiments in Sec.~\ref{sec:vae_experiments}, \ref{sec:weak_sup}, we resized images to be 32 by 32 (note: this is different than the usual setup for MNIST, which uses 28 by 28). During training, we use 40 latent dimensions, a batch size of 64, and the Adam optimizer \citep{kingma2014adam}. We train all models for 200 epochs. For GAN models we take the parameters from the last epoch; for all others, we take the epoch with the highest test objective as measured on the validation set. For VAE and NVP models we use a learning rate of 2e-4; for GAN models we use 2e-4 for the generator and 1e-5 for the discriminator. We found this to be important to help fight modal collapse. For grayscale and color images, we parameterize $p(\vx|\vz)$ with a Bernoulli likelihood. For labels, we use a Bernoulli likelihood for CelebA and a Categorical likelihood otherwise For text, we again use a Categorical likelihood over all words in each sentence.

\paragraph{MNIST}
We use the MNIST hand-written digits (numbers 0 through 9) dataset \citep{lecun1998gradient} with 50,000 examples for training, 10,000 validation, 10,000 testing.

\paragraph{FashionMNIST}
This is an MNIST-like fashion dataset containing grayscale images of clothing from 10 classes---skirts, shoes, t-shirts, etc \cite{xiao2017fashion}. This is intended to be more difficult visually than MNIST. We use the same data splits as MNIST.

\paragraph{CIFAR10}
The CIFAR10 dataset contains 60,000 color images of ten classes (e.g. airplane, dog, ship). We use the same data splits as MNIST. Images in CIFAR10 contain much richer content than MNIST and FashionMNIST. As such, VAEs are known to struggle with sample quality, often generating very blurry images (as evidenced in this paper among many).

\paragraph{CelebA}
The CelebFaces and Attributes (CelebA) dataset \citep{liu2015faceattributes} contains over 200k images of celebrities. Each image is tagged with 40 attributes i.e. wears glasses, or has bangs. We use the aligned and cropped version with a selected 18 visually distinctive (and balanced) attributes, as done in \citep{perarnau2016invertible}.

\paragraph{MNIST Math} Using MNIST, we build an image-captioning dataset using arithmetic expressions that evaluate to a digit between 0 and 9 inclusive. To build a math equation, we sample a number of ``+" operators (call this $k$) between 1 and 5 uniformly. Then we randomly sample $k+1$ operands (digits between 0 and 9). Following this, we rejection sample, throwing away equations that evaluate to greater than 9. Finally, we can randomly match an image in MNIST with an equation that evaluates to its corresponding label. Training and testing splits are kept as in the original MNIST.

\paragraph{Chairs In Context} The Chairs in Context (CIC) dataset contains 78,789 utterances, each produced by a human in conversation with another human. Specifically, this data was collected using a \textit{reference game}: two participants use an online web environment where every trial, participant A is given 3 images of chairs and tasked to describe a target chair to participant B, who apriori does not know which chair is being described. In total, there are 4,511 unique chair renderings, taken from ShapeNet. In their work \citep{achlioptas2018learning}, the authors focus on investigating the varying contexts (the two chairs that were not chosen) on the utterance produced. Here, we ignore context completely, taking only the target chair image and the corresponding utterance.


\paragraph{Flickr30k} The Flickr30k dataset is a standard benchmark dataset for image captioning, containing 158k captions generated from query text used in image search (on Flickr) for 30k images. Multiple captions are used for each image. We treat each instance as an independent observation. Each caption tends to only describe one or two of the subjects (among many) in the scene.

\paragraph{COCO Captions} The Microsoft COCO \citep{lin2014microsoft} dataset contains images of complex everyday scenes of common objects in natural contexts. In total there are 328k images, each with 5 captions describing a subset of objects in the scene. Given the already large size of the dataset, we only use the first caption for each image.

\section{Evaluation}

For image and label datasets, we measure the image marginal and joint log likelihoods:
\begin{align*}
  \mathbb{E}_{p_d^{\text{test}}(\vx)}[\log p_\theta(\vx)] &= \mathbb{E}_{p_d^{\text{test}}(\vx)}[\log \mathbb{E}_{q_\phi(\vz|\vx)}[\frac{p_\theta(\vx|\vz)p(\vz)}{q_\phi(\vz|\vx)}]] \\
  \mathbb{E}_{p_d^{\text{test}}(\vx,\vy)}[\log p_\theta(\vx,\vy)] &= \mathbb{E}_{p_d^{\text{test}}(\vx,\vy)}[\log \mathbb{E}_{q_\phi(\vz|\vx,\vy)}[\frac{p_\theta(\vx|\vz)p_\theta(\vy|\vz)p(\vz)}{q_\phi(\vz|\vx,\vy)}]]
\end{align*}
where the superscript $\text{test}$ indicates an empirical distribution not seen in training.
In addition, for image and text datasets, we measure the text marginal log likelihood by:
\begin{equation}
\mathbb{E}_{p_d^{\text{test}}(\vy)}[\log p_\theta(\vy)] = \mathbb{E}_{p_d^{\text{test}}(\vy)}[\log \mathbb{E}_{q_\phi(\vz|\vy)}[\frac{p(\vz)\prod_{\vw_i \in \vy} p_\theta(\vw_i|\vz,\vec{\vw}_{-i})}{q_\phi(\vz|\vy)}]]
\label{eqn:marginaly}
\end{equation}
where an utterance is made up of a sequence of words $y = (\vw_1, \vw_2, ...)$ and $\vec{\vw}_{-i} = (\vw_1, ..., \vw_{i-1})$ represent the words up to index $i$.
In replacement of classification error, we compute conditional perplexity:
\begin{equation}
\text{perp} = \exp\{-\frac{1}{N^{\text{test}}}\sum_{j=1}^{N^{\text{test}}}\frac{1}{|y_j|}\sum_{w_i \in \vy_j} \log p_\theta(\vw_i|\vz,\vec{\vw}_{-i})\}
\label{eqn:perp}
\end{equation}
where $\vz \sim q_\phi(\vz|\vy)$ and $N^{\text{test}}$ is the size of the test set.

\section{Model Architectures}
\label{sec:arch}
We discuss a few choices for model architectures. Exact implementations will be available online. In choosing these, we prioritized quick training times and relevant designs from previous literature. Future work can explore more complex architectures and greater hyperparameter tuning. For all models, we use ReLU or LeakyReLU. For all image generative networks (in VAEs and GANs) we use a sequence of 3 to 4 de-convolutional layers followed by a convolutional layer and fully connected units. For all image inference networks, we use 3 to 4 convolutional layers. GAN discriminators are similarly designed. GAN models additionally have batch normalization after each convolutional layers (as in DCGAN). Encoding and Decoding networks for labels have 3 linear layers given that this modality is very simple. For all text based networks, we use RNNVAE style encoders and decoders \citep{bowman2015generating}; in this context, word dropout is extremely important to discourage ignoring the latent variable in this autoregressive model. For flow based models, we use the realNVP design \citep{dinh2016density} with ResNets \citep{he2016deep} in each coupling layer.

\section{Training with Unpaired Data}

As referenced by the main text, Alg.~\ref{alg:weaksup} shows how to train a multimodal generative model with unpaired data. In short, sample unimodal data to compute the ``unimodal" terms in the objective. The same algorithm can be used for VAEVAE and VAEGAN models.

\begin{algorithm}[H]
    \caption{{\footnotesize Forward Pass with Unpaired Data}}
    \label{alg:weaksup}

    \begin{algorithmic}
      \STATE {\bfseries Input:} Supervision example $(\vx, \vy)$, Unsupervised example $\vx'$, Unsupervised example $\vy'$
      \STATE $\vz', \vz_{\vx}, \vz_{\vy} = q_\phi(\vz|\vx,\vy), q_\phi(\vz|\vx), q_\phi(\vz|\vy)$
      \STATE $d_1 = D(q_\phi(\vz'|\vx,\vy) || q_\phi(\vz'|\vy))$
      \STATE $d_2 = D(q_\phi(\vz'|\vx,\vy) || q_\phi(\vz'|\vx))$
      \STATE $\mathcal{L}' = \log p_\theta(\vx|\vz') + \log p_\theta(\vy|\vz') + d_1 + d_2$
      \STATE $\mathcal{L}_\vx = \log p_\theta(\vx'|\vz_\vx) + D(q_\phi(\vz_\vx|\vx') || p(\vz))$
      \STATE $\mathcal{L}_\vy = \log p_\theta(\vy'|\vz_\vy) + D(q_\phi(\vz_\vy|\vy') || p(\vz))$
      \STATE $\mathcal{L}_{\text{WS}} = \mathcal{L}' + \mathcal{L}_\vx + \mathcal{L}_\vy$
    \end{algorithmic}
\end{algorithm}

\section{Additional Compositionality Results}
\label{sec:sup:compexp}

In the main text, we used an elaborate procedure to define the oracle representation with the goal being a representation that identifies each individual object in the image. In their original work, \citep{andreas2019measuring} suggested a much simpler framework: Define the vector embeddings as a bag-of-word features from the caption (e.g. there is a blue square and a red triangle) of a ShapeWorld image and sum over them as the composition operator. While straightforward, we found this setup lacked the information to identify individual objects. Critically, these features did not include the size, rotation, and location of objects. Because these features were missing, the bag-of-words could not distinguish between multiple red triangle in image, for example. However, to stay consistent to prior work, we also include results in Table.~\ref{table:tre:supp} of TRE as presented in \citep{andreas2019measuring}.

\begin{table}[h!]
\centering
\begin{tabular}{l|c|c}
    \toprule
    Model & TRE (L$_2$) & TRE (L$_1$)  \\
    \midrule
    VAE & $7.523 \pm 0.034$ & $81.128 \pm 0.396$ \\
    GAN+ALI & $7.513 \pm 0.031$ & $80.910 \pm 0.358$ \\
    VAEVAE & $5.953 \pm 0.033$ & $68.287 \pm 0.346$ \\
    VAEGAN & $\mathbf{2.304 \pm 0.026}$ & $26.382 \pm 0.362$ \\
    VAEGAN$^\S$ & $3.431 \pm 0.021$ & $\mathbf{14.416 \pm 0.104}$ \\
    \bottomrule
\end{tabular}
\caption{Estimating compositionality with TRE score as computed in \citep{andreas2019measuring}.}
\label{table:tre:supp}
\end{table}

\section{Sample Zoo}

Fig.~\ref{fig:imagelabel_samples},~\ref{fig:imagetext_samples} show randomly chosen samples for VAEVAE, VAEGAN, and VAENVP models on a suite of multimodal datasets. For those involving images and text, we show generated captions as well.

\begin{figure}[h!]
  \centering
  \includegraphics[width=\textwidth]{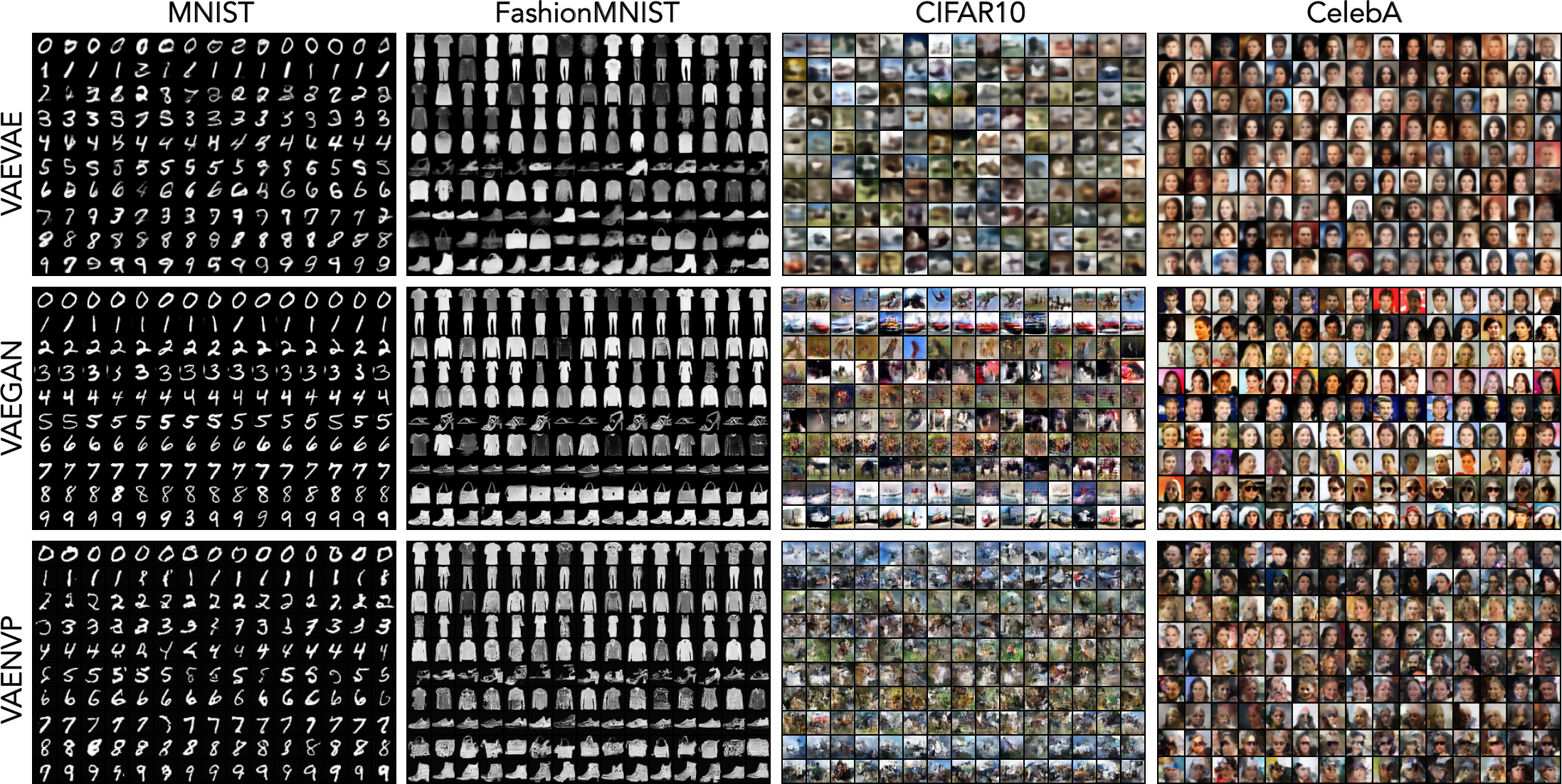}
  \caption{\textit{Image and Label Sample Zoo} -- Samples generated by VAEVAE, VAEGAN, and VAENVP models. Each row contains samples of the same label (e.g. digit number for MNIST or attribute in CelebA). VAEVAE models struggle with blurriness for richer images.}
  \label{fig:imagelabel_samples}
\end{figure}

\begin{figure}[h!]
  \centering
  \begin{subfigure}[h]{\textwidth}
      \centering
      \caption{Marginal Samples}
      \includegraphics[width=\linewidth]{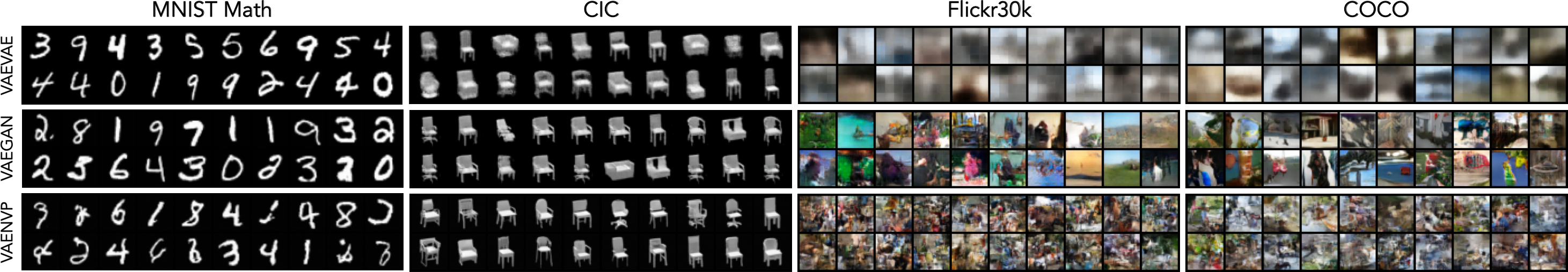}
  \end{subfigure}
  \begin{subfigure}[h]{\textwidth}
      \centering
      \caption{Conditional Samples}
      \includegraphics[width=\linewidth]{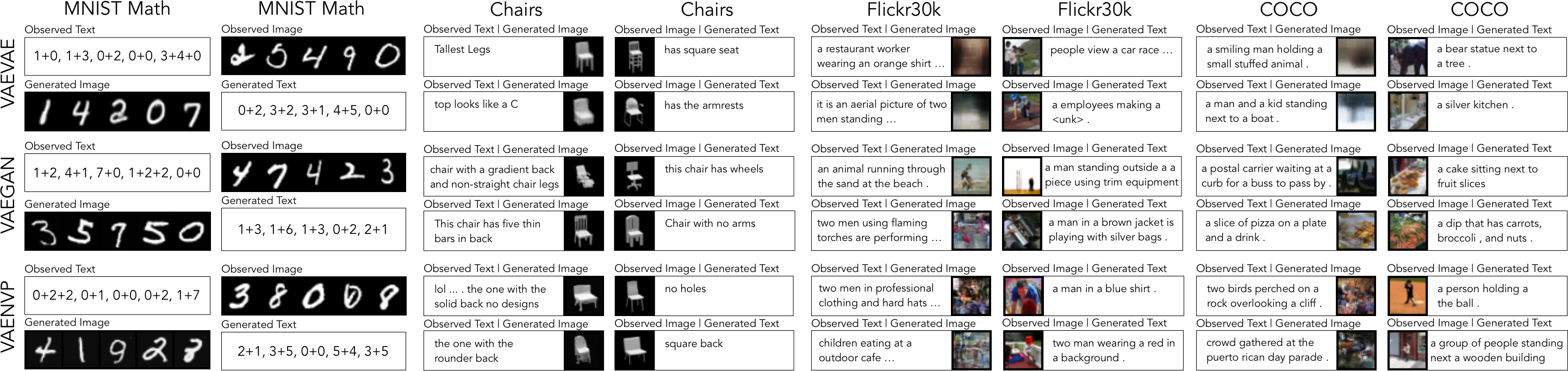}
  \end{subfigure}
  \caption{\textit{Image and Text Sample Zoo} -- Marginal samples and conditional samples generated by VAEVAE, VAEGAN, and VAENVP models. VAEGAN significantly outperforms other models in sample quality for more complex image domains.}
  \label{fig:imagetext_samples}
\end{figure}

\section{Details on Downstream Image Experiments}

\paragraph{Object Detection} When training VAEVAE or VAEGAN, we use 100 latent dimensions and reshape images to be a fixed 64 by 64 pixel size. As in previous experiments, we use a learning rate of 2e-4 and train for 200 epochs with a batch size of 64 and the Adam optimizer. We use an RNNVAE for the text generative model and a convolutional network for the image generative model. These are the same as describe in Sec.~\ref{sec:arch}.

\paragraph{Visual Relation Detection} Because the images in Visual Genome are in general more complex, we use 256 latent dimensions and reshape images to be 64 by 64 pixels. All other choices are as in object detection.

\section{Using Representations in Reference Games}
In this section of the Appendix, we investigate a possible use case of multimodal representations in grounded language learning (not presented in the main text). Future work could explore this direction in more detail.

\begin{figure}[h!]
  \centering
  \includegraphics[width=\linewidth]{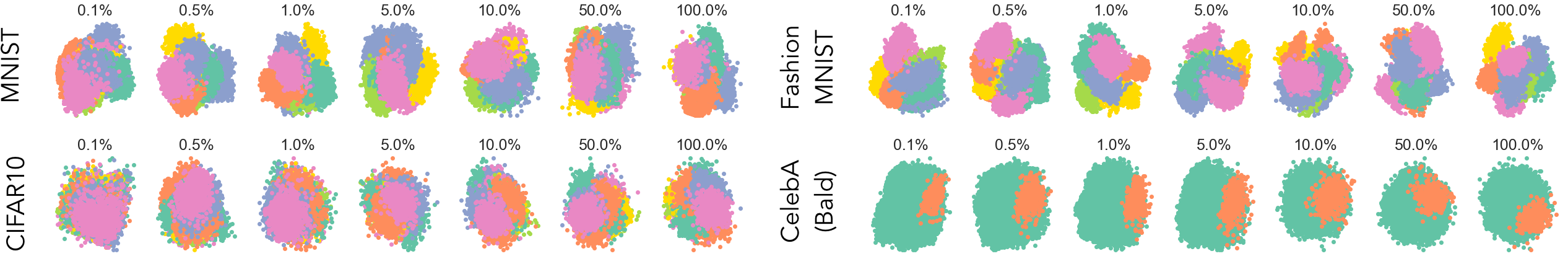}
  \caption{PCA projections of VAEVAE (joint) embeddings for varying levels of supervision across four image and label datasets. Increased supervision tends towards more well defined clusters.}
  \label{fig:rep24}
\end{figure}

One of the attractive properties of a generative model is a learned (latent) representation for a given image. Likewise, we can obtain a joint representation using a multimodal generative model. Specifically, we take the mean of the variational posterior $q_\phi(z|x,y)$ as a ``joint" encoding of $x$ and $y$. For instance, Fig.~\ref{fig:rep24} shows a 2-dimensional PCA projection of embeddings from the test set of four labelled image datasets as the amount of supervision is varied from low (0.1\%) to high (100\%). We note that the separation between classes increases (more defined clusters) as supervision increases. In particular, even with 5\% supervision (lower for simpler decision boundaries), we start to see distinct clusters form, suggesting that we can efficiently compress the underlying structure of the data domain without over-dependence on supervision. These embeddings may be useful for downstream predictive tasks. We find similar behavior for VAEGAN models.

\begin{figure}
  \centering
  \begin{subfigure}[h]{0.40\textwidth}
      \caption{}
      \includegraphics[width=0.8\textwidth]{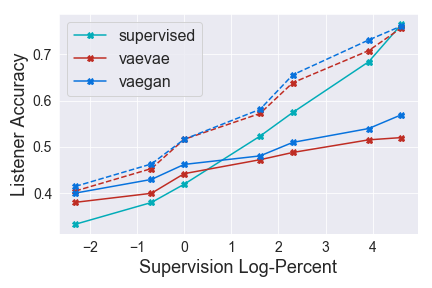}
  \end{subfigure}
  \begin{subfigure}[h]{0.51\textwidth}
      \caption{}
      \includegraphics[width=\linewidth]{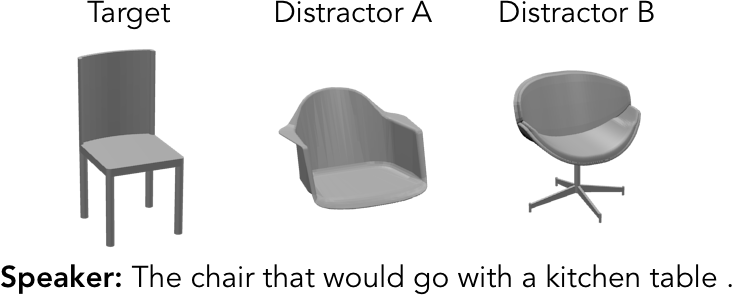}
  \end{subfigure}
  \caption{(a) Weak supervision experiment for  listeners in a Chairs reference game. (b) Example of a target chair (chosen by listener), two distractor chairs, and a speaker utterance.}
  \label{fig:refgame}
\end{figure}

One of the attractive properties of our generative model is a multimodal representation that may be used for downstream tasks. For example, we can apply
the learned representation to reference games \citep{achlioptas2019shapeglot,cohn2018pragmatically,hawkins2017convention} as used in computational cognitive science. Reference games are an experiment tool for studying human pragmatic reasoning in a controlled setting: there are two agents: a listener and a speaker. Given a context of 3 possible images, the speaker's task is to describe one of them. The listener's role is to guess the speaker's intended referent from the set of three. A large branch of research studies how to model the listener and speaker as Bayesian computational models \citep{goodman2016pragmatic}. More recent work \citep{cohn2018pragmatically} has instead parameterized the listener and speaker as deep neural networks. Within reference games, paired data is collected by human agents using Mechanical Turk, meaning that common dataset sizes are limited to a few thousand at best. This is not enough to train many deep models and therefore, a perfect use case for weakly supervised learning.

Here, we begin to study applying multimodal generative models to model listener behavior. To gauge performance, we measure how well a computational model can choose the ``correct" image out of three given a speaker utterance as compared to human listener choices. In particular, the Chairs in Context (CIC) dataset was released as a reference game dataset. In our experiments above, we only used the utterance and target image but the full dataset contains a speaker utterance, a set of three chair images, and the (``target") chair chosen by a listener. We only consider set of images in the \textit{far} context, meaning the three chairs have visually distinctive features. Future work should explore the more difficult \textit{close} context. In Fig.~\ref{fig:refgame}, we compare the accuracy of a supervised neural model as used in reference game literature (trained specifically for the reference game task) and a VAEVAE and VAEGAN as we vary the amount of paired data. With full supervision, the dataset contains 67k examples. To make a prediction for which chair the speaker is referencing using a generative model, we take $\arg\max_{i \in \{1,2,3\}} p_\theta(\vx=\vx_i|\vz=\mathbf{\mu})$ where $\mathbf{\mu}$ is the mean for posterior $q_\phi(\vz|\vy)$. The set $\{\vx_1, \vx_2, \vx_3\}$ represents the images in context and $\vy$ represents the speaker utterance. Fig.~\ref{fig:refgame} shows that a VAEVAE and VAEGAN outperform a supervised model in the low data regime (around 100 to 1000 examples). Beyond 1000 examples, supervised networks specifically trained as listeners in a reference game out-perform the generative models. In some sense, it is not fair to compare these generative models and the supervised model since only the latter was trained with knowledge of the reference game setup. Perhaps a more fair comparison would be to use the unsupervised embeddings learned by VAEVAE or VAEGAN and fit a small network to play the reference game.
The dotted lines in Fig.~\ref{fig:refgame} show exactly this; we can see the performance is much higher than the generative models alone and higher than the supervised baseline as well. Future work should investigate other reference games \citep{monroe2017colors} as well as possibility of using unpaired text and images for more efficient learning. The ability to learn listeners with less data could save time and resources for many future cognitive studies.

\end{document}

%% file: math_commands.tex
\usepackage{amsmath,amsfonts,bm}

\def\eqref#1{equation~\ref{#1}}

\def\1{\bm{1}}

\def\vw{{\bm{w}}}
\def\vx{{\bm{x}}}
\def\vy{{\bm{y}}}
\def\vz{{\bm{z}}}

\DeclareMathAlphabet{\mathsfit}{\encodingdefault}{\sfdefault}{m}{sl}
\SetMathAlphabet{\mathsfit}{bold}{\encodingdefault}{\sfdefault}{bx}{n}

\newcommand{\KL}{D_{\mathrm{KL}}}

